\documentclass[nohyperref]{article}

\usepackage{microtype}
\usepackage{graphicx}
\usepackage{subfigure}
\usepackage{booktabs} 

\usepackage{hyperref}

\usepackage[accepted]{./Template/Format}

\usepackage{amsmath}
\usepackage{amssymb}
\usepackage{mathtools}
\usepackage{amsthm}

\usepackage[capitalize,noabbrev]{cleveref}

\usepackage[utf8]{inputenc} 
\usepackage[T1]{fontenc}    
\usepackage{hyperref}       
\usepackage{url}            
\usepackage{booktabs}       
\usepackage{amsfonts}       
\usepackage{nicefrac}       
\usepackage{microtype}      
\usepackage{xcolor}         
\usepackage{graphicx}
\usepackage{wrapfig}
\usepackage{mathtools}
\usepackage{amsthm}
\usepackage{enumitem}
\usepackage{natbib}
\usepackage{algorithm}
\usepackage{xcolor}
\definecolor{myred}{HTML}{800606}
\definecolor{myblue}{HTML}{131D85}
	
\usepackage{color, colortbl}
\usepackage{pifont}

\theoremstyle{plain}
\newtheorem{theorem}{Theorem}[section]
\newtheorem{proposition}[theorem]{Proposition}
\newtheorem{lemma}[theorem]{Lemma}

\theoremstyle{definition}
\newtheorem{definition}[theorem]{Definition}

\theoremstyle{remark}
\newtheorem{remark}[theorem]{Remark}

\definecolor{lightpurple}{rgb}{0.96, 0.96, 1}
\newcommand{\T}{\mathbb{T}}
\newcommand{\Z}{\mathbb{Z}}
\newcommand{\R}{\mathbb{R}}
\newcommand{\cA}{\mathcal{A}}

\newcommand{\cF}{\mathcal{F}}
\newcommand{\cG}{\mathcal{G}}

\newcommand{\cK}{\mathcal{K}}
\newcommand{\cL}{\mathcal{L}}

\newcommand{\cN}{\mathcal{N}}

\newcommand{\cP}{\mathcal{P}}
\newcommand{\cQ}{\mathcal{Q}}

\newcommand{\cU}{\mathcal{U}}

\newcommand{\bC}{\mathbb{C}}

\newcommand{\bN}{\mathbb{N}}

\newcommand{\bR}{\mathbb{R}}

\newcommand{\bT}{\mathbb{T}}

\newcommand{\dA}{\cA\left(\Omega;\bR^{d_a}\right)}
\newcommand{\dU}[1][u]{\cU\left(\Omega;\bR^{d_#1}\right)}
\newcommand{\dV}[1][u]{\cU\left(\Omega;\bR^{d_v}\right)}
\newcommand{\dVt}[1][u]{\cU\left(\left[0,\infty\right)\times\Omega;\bR^{d_v}\right)}
\newcommand{\dCTU}[1][u]{C\left([0,T]\times\dU[#1]\right)}

\newcommand{\dH}[2][s]{H^{#1}#2}
\newcommand{\dCTH}[2][s]{C\left([0,T]\times \dH[#1]{#2}\right)}

\newcommand{\set}[2]{{\left\{ #1 \,\middle|\, #2 \right\}}}
\newcommand{\slot}{{\,\cdot\,}}

\newcommand{\eps}{\epsilon}
\newcommand{\norm}[2][H^s]{\left\|#2\right\|_{#1}}

\usepackage[textsize=tiny]{todonotes}

\Titlerunning{Learning PDE Solution Operator for Continuous Modeling of Time-Series}

\begin{document}

\onecolumn
\icmltitle{Learning PDE Solution Operator for Continuous Modeling of Time-Series}

\setsymbol{equal}{*}

\begin{Authorlist}
\Author{Yesom Park}{equal}
\Author{Jaemoo Choi}{equal}
\Author{Changyeon Yoon}{equal}
\Author{Chang hoon Song}{}
\Author{Myungjoo Kang}{}
\end{Authorlist}
\Info
\vspace{20pt}

\correspondingauthor{Myungjoo Kang}{mkang@snu.ac.kr}
\let\thefootnote\relax\footnotetext{\textsuperscript{*}Equal contribution authors. Correspondence to: \tt <mkang@snu.ac.kr>.}

\begin{abstract}
Learning underlying dynamics from data is important and challenging in many real-world scenarios.
Incorporating differential equations (DEs) to design continuous networks has drawn much attention recently, however, most prior works make specific assumptions on the type of DEs, making the model specialized for particular problems. 
This work presents a partial differential equation (PDE) based framework which improves the dynamics modeling capability. 
Building upon the recent Fourier neural operator, we propose a neural operator that can handle time continuously without requiring iterative operations or specific grids of temporal discretization. A theoretical result demonstrating its universality is provided.
We also uncover an intrinsic property of neural operators that improves data efficiency and model generalization by ensuring stability.
Our model achieves superior accuracy in dealing with time-dependent PDEs compared to existing models. 
Furthermore, several numerical pieces of evidence validate that our method better represents a wide range of dynamics and outperforms state-of-the-art DE-based models in real-time-series applications.
Our framework opens up a new way for a continuous representation of neural networks that can be readily adopted for real-world applications.
\end{abstract}

\section{Introduction}
The modeling of time-evolving data plays an important role in various applications in our everyday lives, including climate forecasting \cite{schneider2001analysis, mudelsee2019trend}, medical sciences \cite{stoffer2012special, jensen2014temporal}, and finance \cite{chatigny2020financial, andersen2005volatility}.
Numerous deep learning architectures \cite{connor1994recurrent, hochreiter1997long, cho2014properties} have been developed to learn sequential patterns from time-series data.
In recent years, leveraging differential equations (DEs) to design continuous networks has attracted increasing attention, first sparked by neural ordinary differential equations (Neural ODEs; \citealt{chen2018neural}).
DEs that characterize the rates of change and interaction of continuously varying quantities have become the indispensable mathematical language to describe time-evolving real-world phenomena \cite{cannon2012evolution, sunden2016heat, black2019pricing}. 
By virtue of their ability to represent and predict the world around us, incorporating DEs into neural networks has reinvigorated research in continuous deep learning, offering the ability to handle irregular time-series \cite{rubanova2019latent, de2019gru, schirmer2022modeling}.

Despite their eminent success, Neural ODEs have yet to be successfully applied to complex and large-scale tasks due to the limitation of expressiveness of ODEs.
To respond to this drawback, several works have enhanced the expressiveness of Neural ODEs \cite{gholami2019anode, gu2021efficiently}.
Another line of work attempts to introduce more diverse differential equations, such as controlled DEs \cite{kidger2020neural}, delay DEs \cite{zhu2021neural, anumasa2021delay}, and integro DEs \cite{zappala2022neural}.
However, in real-world applications, we usually know very little about the underlying dynamics of time-evolving systems, such as how their states evolve in general and which differential equations they obey, how variables depend on each other, 
and how high derivatives it contains. Therefore, it is necessary to develop a model that can learn an extended class of differential equations that is able to cover more diverse applications \cite{holt2022neural}.

In this work, we propose a partial differential equation (PDE) based novel framework that can learn a broad range of time-evolving systems without prior knowledge of governing equations.
PDEs that enjoy relations between the various partial derivatives of multivariable states represent much general dynamics, including ODEs as a special case. 
As the underlying dynamics are unknown in real-world data, it should be oblivious to the knowledge of the underlying PDE structure and needs to be learned from the data.
To this end, we adopt Fourier neural operator (FNO; \cite{li2020fourier}) that 
automatically learns PDE solution operators in a completely data-driven way without prior information on the governing PDE.
Because FNO handles time in discrete representation, however, FNO is difficult to directly transfer to irregularly-sampled time-series commonly arising in real-world problems.
To render it more suitable for continuous time-series, we propose a continuous-time FNO, termed \emph{CTFNO}, that can treat time continuously without requiring a specific temporal grid.
We also demonstrate the representational power of CTFNO via rigorous theoretical proof of the universal approximation theorem.
Moreover, we present a property of neural operator that guarantees stability. As it leads to well-posed learning problems, the stabilization makes to model better at generalization.
A wide array of numerical evidence validates that CTFNO can flexibly capture diverse time-dependent systems, outperforming baseline models not only for PDEs but also for various dynamics.
Furthermore, our model provides superior performance on a wide array of real-world time-series data.

\section{Background}
\label{prelim}
\paragraph{Fourier Neural Operator} \label{sec:fno}
Let $\Omega\subset\mathbb{R}^n$ be a bounded domain. For a given input $a:\Omega\rightarrow \bR^{d_a}$, which could be any of source or initial functions, neural operators learn the corresponding solution $u:\Omega\rightarrow \bR^{d_u}$ to a governing PDE.
The solution to fairly general PDEs is represented as a convolution operator with a kernel $G:\mathbb{R}^n\rightarrow \mathbb{R}^{d_u\times d_a}$ called by a Green's function \cite{evans2010partial} as follows:         
\begin{equation}
u\left(x\right)=G\ast a\left(x\right)=\int_{\Omega}G\left(x-y\right)a\left(y\right)dy,\ \forall x\in \Omega.\label{eq:Green}
\end{equation}
Due to the shift-invariant nature of the Green's function, this solution operator can be efficiently computed through the Fourier transform, known as the convolution theorem \cite{bracewell1986fourier}.
This elucidates a way to design Fourier neural operator (FNO; \citealt{li2020fourier}).
The overall computational flow of FNO for approximating the convolution operator (\ref{eq:Green}) is given as 
\vspace{-2pt}
\[
a\overset{\cP}{\longrightarrow}v_0\overset{\cL_1}{\longrightarrow}v_1
\overset{\cL_2}{\longrightarrow}\cdots \overset{\cL_L}{\longrightarrow}v_L\overset{\cQ}{\longrightarrow}u
\]
for a given depth $L$. 
To increase expressiveness, the input function $a$ is lifted to a higher dimensional representation by
$v_{0}=\cP(a)(x)\coloneqq Pa(x)$ with a matrix $P\in\bR^{d_v\times d_a}$.
$\cQ$ is a projection operator of the form $\cQ(v)(x)\coloneqq Qv(x)$ for $Q\in\bR^{d_u\times d_v}$.
Fourier layers $\cL_\ell$ are defined as follows:
\begin{definition}\label{def:FNO}
(\textbf{Fourier layers} \cite{li2020fourier}) For a convolution kernel $\kappa_\ell:\bR^n\rightarrow\bR^{d_v\times d_v}$, a linear transform $W_\ell:\mathbb{R}^{d_v}\rightarrow\mathbb{R}^{d_v}$, and an activation function $\sigma :\bR\rightarrow\bR$, the $\ell$-th Fourier layer $\cL_\ell$ is defined as follows:
\begin{align}
\mathcal{L}_{\ell}\left(v\right)\left(x\right)&\coloneqq\sigma\left(W_{\ell}v\left(x\right)+k_{\ell}\ast v\left(x\right)\right) \\
& =\sigma\left(W_{\ell}v\left(x\right)+\mathcal{F}^{-1}\left(R_{\ell}\cdot\left(\mathcal{F}v\right)\right)\left(x\right)\right),\ \forall v:\Omega\rightarrow\bR^{d_v},\ x\in\Omega,
\end{align}
where $R_{\ell}=\mathcal{F}\left(\kappa_{\ell}\right)$ is directly learned, $\cF$ is Fourier transform, of which the inverse operator is denoted by $\cF^{-1}$.
\end{definition}

Both $\cF$ and $\cF^{-1}$ are implemented by fast Fourier transform \cite{nussbaumer1981fast} with truncated frequencies.

\paragraph{Treatment of time-varying problems}
When applied to time-dependent PDEs, the original FNO can only learn an operator that maps the initial function to a solution for a single fixed time. 
To deal with time-varying problems, two methods are suggested in \cite{li2020fourier}: FNO-RNN poses the time-dependent problem as a sequence-to-sequence task. 
But autoregressive training is often hard to train.
Alternatively, FNO-2D treats it as an $(n+1)-$dimensional problem by adding one more dimension and applying FNO layers to convolve in the space-time domain. 
In this case, the model can only predict solutions at times on a fixed equispaced temporal mesh.
In addition, it requires quite a few parameters because, for example, a one-dimensional problem is treated as a two-dimensional problem.

\section{Continuous-Time PDE Solution Operator} \label{sec:CTFNO}
\begin{figure*}[t]
    \centering
    \includegraphics[page=1,width=0.90\textwidth]{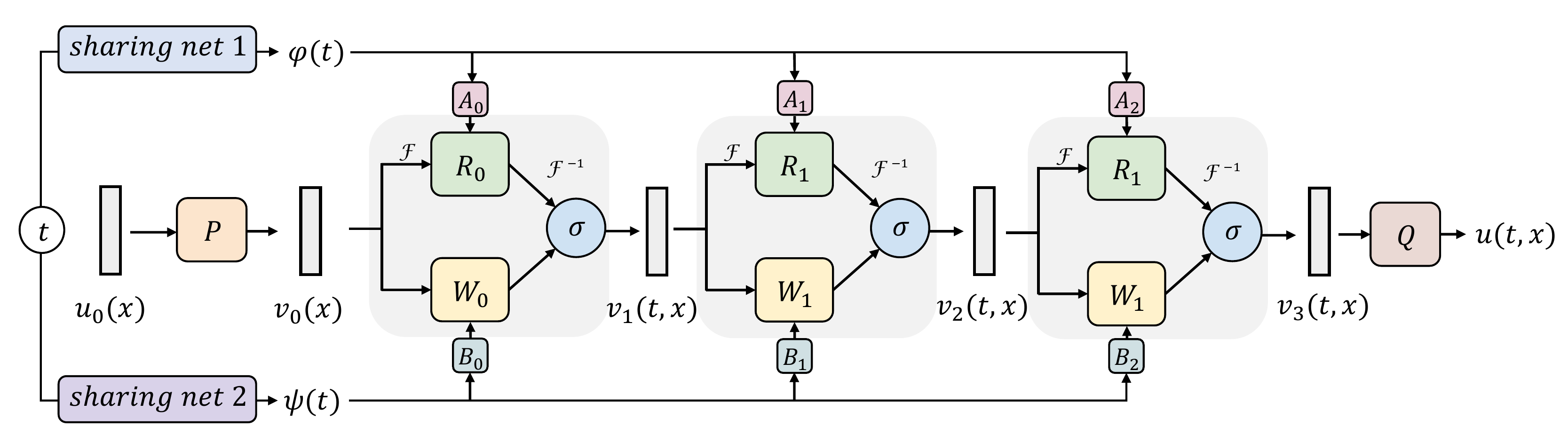}
    \caption{The visualization of CTFNO architecture. Two time embedding networks specify the time $t$ into hidden representations $\varphi(t)$ and $\psi(t)$.
    Then $\varphi(t)$ and $\psi(t)$ incorporate the temporal information into each Fourier layer through time modulation operators $A_\ell$ and $B_\ell$. See Appendix \ref{appen: network_detail} for details.}
    \label{fig:CTFNO}
    \vspace{-5pt}
\end{figure*}
\subsection{Continuous-Time FNO} \label{sec:tfno}
FNO has shown a promising ability to learn complex PDEs, however, it deals with time-evolving systems with iterative rollouts or on specific temporal grids as we discussed in the previous section. This notion of discrete time of FNO hinders its wider applicability to time-varying systems.
To ameliorate this limitation, we introduce a \textit{continuous-in-time Fourier neural operator (CTFNO)}.
The design of CTFNO is inspired by the Green's function formula for time-dependent PDEs, which says that there exists a Green's function $G:\left[0,\infty\right)\times\mathbb{R}^n\rightarrow \mathbb{R}^{d_u\times d_a}$ such that the solution for an initial condition $a\left(x\right)$ is represented by follows: 
\begin{equation}
u\left(t,x\right)=\int_{\Omega}G\left(t,x-y\right)a\left(y\right)dy,\ \forall \left(t,x\right) \in \left[0,\infty\right)\times\Omega.\label{eq:time_Green}
\end{equation}
For example, Green's function of the heat equation $u_{t}-\nu u_{xx}=0$, describing the temperature on a surface as a function of time, is $G\left(t,x-y\right)=\frac{1}{\sqrt{4\pi \nu t}}\exp\left(-\frac{\mid x-y\mid^2}{4\nu t}\right)$.
This shows that the \textbf{weights
of an FNO layer should be conditioned on time}, to learn an operator $\left(t,a\left(x\right)\right)\mapsto u\left(t,x\right)$
for an arbitrary time $t$ and initial condition $a\left(x\right)$. 
To this end, we propose a time-aware Fourier layer that updates the solution based on \eqref{eq:time_Green} as follows:
\begin{definition} \label{def:CTFNO}
\textbf{(Continuous-time Fourier layers)}
For $t\in[0,\infty)$, 
a convolution kernel $\kappa_\ell\left(t\right):\bR^n\rightarrow\bR^{d_v\times d_v}$, $R_{\ell}\left(t\right)=\mathcal{F}\left(\kappa_{\ell}\left(t\right)\right)$, a linear transform $W_\ell\left(t\right):\mathbb{R}^{d_v}\rightarrow\mathbb{R}^{d_v}$,  $\varphi_\ell\left(t\right):\bR^n\rightarrow\bC$, and $\psi_\ell\left(t\right)\in\mathbb{R}^{d_v\times d_v}$
the $\ell$-th continuous-time Fourier layer is defined as follows: $\forall v:\Omega\rightarrow\bR^{d_v},\ x\in\Omega$,
\begin{equation}  \label{eq:CTFNO}
    \begin{aligned}
   \mathcal{L}_{\ell}\left(v\right)\left(t,x\right)
    & =\sigma\left(W_\ell\left(t\right)v\left(x\right)+k_{\ell}\left(t\right)\ast v\left(x\right)\right)\\
     & =\sigma\left(W_\ell\left(t\right)v\left(x\right)+\mathcal{F}^{-1}\left(R_{\ell}\left(t\right)\cdot\left(\mathcal{F}v\right)\right)\left(x\right)\right) \\
     & = \sigma\left(W_\ell\psi_{\ell}(t) v\left(x\right)+\mathcal{F}^{-1}\left(\varphi_\ell\left(t\right)R_{\ell}\cdot\left(\mathcal{F}v\right)\right)\left(x\right)\right).
    \end{aligned}
\end{equation}
\end{definition}

\paragraph{Time Modulation}
To equip the FNO network with the ability to capture information on the time of observations, a time-dependent layer \eqref{eq:CTFNO} is constructed as the following temporal modulation: For each time $t$ and frequency $\xi$,
\begin{equation} \label{eq:t_modulation}
\begin{cases}
W_\ell\left(t\right) & = W_\ell\psi_{\ell}(t),\\
R_\ell\left(t,\xi\right) & =\varphi_\ell\left(t,\xi\right)R_\ell\left(\xi\right),
\end{cases}
\end{equation}
where we use notational shortcuts $\varphi_\ell\left(t,\xi\right)=\varphi_\ell\left(t\right)\left(\xi\right)$ and similar for $R(t,\xi)$. We design $\varphi_\ell$ and $\psi_\ell$ as follows:
\begin{enumerate}[topsep=5pt]
\item Two sharing networks, parameterized by two-layer fully connected networks together with sinusoidal embedding \cite{vaswani2017attention}, first convert the input time $t$ into multi-dimensional representations $\varphi\left(t\right),\ \psi\left(t\right)\in\bR^c$ for a hidden dimension $c$.

\item For each Fourier layer $\mathcal{L}_{\ell}$ and frequency $\xi$,
learnable $A_\ell:\bR^n\rightarrow\bC^c$ and  $B_\ell\in\bR^{d_v\times c}$ produce time information $\varphi_\ell\left(t,\xi\right)=\varphi\left(t\right)^TA_\ell\left(\xi\right)\in\bC$ and $\psi_\ell\left(t\right)=\text{diag}\left(B_\ell\psi\left(t\right)\right)\in\bR^{d_v\times d_v}$ (the diagonal matrix with the elements of vector $B_\ell\psi\left(t\right)$ on the main diagonal).
\end{enumerate}

See \cref{fig:CTFNO} for a schematic diagram.

\subsection{Universal Approximation} \label{sec:UnivApprox}
In this section, we prove the universality of the proposed CTFNO, which is condensed in the following informal statement. A formal statement and details of the proof are provided in \cref{appen:universal}.
\begin{theorem} \label{thm:universal}
\emph{\textrm{\textbf{(Informal)}}} CTFNO can approximate any time-dependent continuous operator, of arbitrary accuracy.
\end{theorem}

\subsection{Stability} \label{sec:Stab}
When we learn a model for the time evolution of dynamical systems, the learned dynamics should be guaranteed to be well-posed.
It is important because the learned system can be unstable when using a generic neural network \cite{szegedy2013intriguing, moosavi2017universal}.
Such unstable networks are vulnerable to adversarial attacks, overfitting, and unauthorized exploitation, which may render the network useless in practice.
Therefore, stability is a necessary condition in real-world applications.
The stability is measured by the sensitivity of the prediction with respect to small perturbations of the inputs \cite{hadamard1902problemes}.
A formal definition is given as follows.
\begin{definition}
(Stability) A time-dependent PDE is said to be \textit{stable} if for any solution $u\left(t,x\right)$ with initial condition $u_{0}\left(x\right)$ and $\epsilon>0$, there exists $\delta>0$ such that for all new initial function $\tilde{u}_{0}\left(x\right)$ satisfying $\left\Vert \tilde{u}_{0}-u_{0}\right\Vert _{\mathcal{}}<\delta$, the corresponding solution $\tilde{u}\left(t,x\right)$ satisfies $\left\Vert \tilde{u}-u\right\Vert <\epsilon$ for all $t\geq0$.
\end{definition}

The stability is a hard constraint imposed upon the model.
While some studies have addressed the stability of network architectures, it has typically been used as a soft constraint by adding an extra regularization loss
\cite{moosavi2019robustness}, or required computation of the eigenvalues of the Jacobian matrix \cite{ross2018improving, hoffman2019robust}.
Besides, the stability of Neural ODEs is elicited by stable discretization techniques of the ODEs \cite{haber2017stable,yan2019robustness}.
The stability of CTFNO is connected to the well-posedness of the learned solution operator.
The kernel formulation (\ref{eq:time_Green}) sheds light on a way to ensure stability. Proof is deferred to Appendix \ref{appen:stable}.

\begin{proposition} \label{prop:stab}
\label{prop:stability_FNO}\textbf{(Stability of CTFNO)} If $\left\Vert R\left(t,\xi\right)\right\Vert _{2}$ and $\left\Vert W\left(t\right)\right\Vert _{2}$ are bounded for every $t>0$, then the corresponding CTFNO with a Lipschitz continuous activation function is stable.
\end{proposition}

\paragraph{Gershgorin discs normalization}
As given in proposition \ref{prop:stability_FNO}, the global stability of CTFNO is guaranteed if the Fourier kernel and weight have bounded $L^2$ norms. Because of the expensive computational cost of $L^2$ norm, however, we suggest a practical method for enforcing stability conditions.
Gershgorin's circle theorem \cite{varga2010gervsgorin} allows us to make fast deductions on the bound of eigenvalues. It states that every eigenvalue $\lambda$ of a square matrix $A=\left(a_{ij}\right)$ satisfies $\left|\lambda-a_{ii}\right|\leq\Sigma_{j\neq i}\left|a_{ij}\right|$ for each $i$. 
Therefore, as we regulating the $L^1$ norm of each row $\mathbf{r}_i$, we can impose the requisite stability.
In implementation, we normalize $\parallel \mathbf{r}_i\parallel_{L^1}\leq M$ for each $i$ with a pre-defined $M>0$.

\section{Experiments}
\subsection{Experiments for learning time-dependent PDEs}\label{sec:pde}
In this section, we empirically validate the performance of the proposed model as a continuous-time neural PDE surrogate. Given an initial function $u_0$, we train models to learn $\left(t,u_{0}\left(x\right)\right)\mapsto u\left(t,x\right)$ for $t\in\left(0,T\right]$ with mean squared error (MSE) loss. Throughout all experiments, we run models three times with different random seeds and report the averaged value.

\paragraph{Datasets} We choose four PDEs for numerical experiments. We consider \textbf{heat} \cite{baron1878analytical} and \textbf{Burgers}' equations \cite{bateman1915some}, which are canonical time-dependent linear and nonlinear PDEs, respectively. They take the form 
\begin{equation}
\frac{\partial u}{\partial t} + \alpha u\frac{\partial u}{\partial x}=\nu \frac{\partial^2 u}{\partial x^2}, \ \  x\in \left(0,1\right),\ t\in\left(0,T\right], \label{eq:heat_and_burgers}
\end{equation}
with the corner cases: heat $\alpha=0$ and Burgers' equation $\alpha=1$. Here, $\nu$ is a positive constant.
We also apply our model for two examples provided by \textsc{PDEBench} \cite{takamoto2022pdebench}: \textbf{compressible Navier-Stokes equations} equations, which describes the motion of fluid dynamics, and \textbf{diffusion-sorption} equation. Diffusion-sorption equation is a diffusion process influenced by a retardation factor, which is is a variable stands for the degree to which the diffusion process is hindered by the sorption interactions.
Detailed descriptions are provided in \cref{sec:data}.

\paragraph{Baselines}
We compare the performance of the proposed model with representative PDE surrogates. DeepONet (DON; \citealt{lu2019deeponet}) is an alternative operator learning method that represents the solution operator by a basis expansion. POD-DeepONet (PDN; \citealt{lu2022comprehensive}), a model  based on a proper orthogonal decomposition of function spaces, is also considered. FNO-2D \cite{li2020fourier} is a Fourier neural operator with spatio-temporal inputs and FNO-RNN is an autoregressivly trained FNO.
 
\begin{table}
    \centering
      \setlength\tabcolsep{18.0pt}
    \caption{RMSE ($\times10^{-2}$) results and the number of parameters of each model on PDE problems.} \label{tab:PDEs}
     \scalebox{1.0}{
    \begin{tabular}{cccccc}
    \toprule
    Model & Heat & Burgers & Diffusion-Sorption & Navier-Stokes & \# Params \\
    \midrule 
    FNO-RNN &  9.506 & 73.100 & 5.187 & 36.791 & 2.02M \\
    FNO-2D      & 0.033  & 3.136 & 0.053 & 4.486 & 7.00M \\
    DON        & 0.473  & 6.022 & 0.234 & 4.207 & 1.58M \\
    PDN        & 0.323  & 5.796 & 0.314 & 3.781 & 1.78M \\
    \rowcolor{lightpurple}
    CTFNO         & \textbf{0.026}  & \textbf{1.952} & \textbf{0.042} & \textbf{2.947} & 2.38M \\
    \bottomrule
  \end{tabular}}
    \vspace{-5pt}
\end{table} 

\paragraph{Results}
\begin{wraptable}{r}{0.32\textwidth}
    \centering
    \setlength\tabcolsep{5.0pt}
    \vspace{-10pt}
    \caption{Training and inference time (second/epoch) on heat equation.} \label{tab:Time_elapsed_heat}
    \vspace{5pt}
        \begin{tabular}{ccc}
        \toprule
        Model & Training & Inference \\
        \midrule
        FNO-RNN  & 6.34 & 0.74 \\
        FNO-2D & 2.06 & 0.30 \\
        \rowcolor{lightpurple}
        CTFNO & \textbf{0.79} & \textbf{0.05} \\
        \bottomrule
        \end{tabular}
    \vspace{-5pt}
\end{wraptable}
Results of the test root MSE (RMSE) are reported in \cref{tab:PDEs}.
We can see that CTFNO significantly outperforms all baselines.
Comparing the results of CTFNO with the original FNOs, the core strengths of the proposed model stand out more.
The results show that the autoregressive learning-based FNO-RNN is difficult to capture the dynamics of PDEs accurately. Also, it requires several autoregressive rollouts to predict the solution after a long time, which is rather time-consuming.
On the other hand, CTFNO can predict the solution with a single call. 
Besides, unlike FNO-2D, which can only predict solutions at times on a fixed uniform grid, CTFNO can predict a solution at any desired time, retaining the number of parameters regardless of the length of prediction time.
\cref{tab:PDEs} shows that our model uses five times fewer parameters than FNO-2D.
The results demonstrate that the use of the proposed time-dependent structure significantly improves the capacity of the model to treat time.
Furthermore, we obtain 2–8$\times$ and 6-15$\times$ speed-ups for training and inference time, respectively (See \cref{tab:Time_elapsed_heat}).
Moreover, our model is superior to existing benchmark PDE models.
The results confirm that our model describes the diffusion phenomenon quite well.
Furthermore, CTFNO outperforms other models even in the dissipative nonlinear system with shock formation (Burgers), and complex fluid dynamics (Navier-Stokes).
We also include additional heatmaps of the learned solution compared with the exact solution over the entire time in Appendix \ref{appen:heatmaps_ode_heat_burgers}.
The overall results confirm the superiority of CTFNO over existing PDE surrogates for learning time-dependent PDEs.
\paragraph{Ablation study on where to assign time}
\begin{wraptable}{r}{0.32\textwidth}
    \centering
    \vspace{-20pt}
      \setlength\tabcolsep{9.0pt}
    \caption{RMSE ($\times10^{-2}$) errors of ablation studies.} \label{tab:ablation}
    \vspace{5pt}
     \scalebox{1.0}{
    \begin{tabular}{ccc}
    \toprule
    Model & Heat & Burgers\\
    \midrule 
    Baseline 1  & 1.538 & 5.654\\
    Baseline 2    & 0.370 & 5.305 \\
    Baseline 3    &  0.794 & 4.426 \\
    \rowcolor{lightpurple}
    CTFNO         & \textbf{0.026}  & \textbf{1.952}\\
    \bottomrule
  \end{tabular}}
    \vspace{-55pt}
\end{wraptable} 
The rationale of the way of imposing temporal information on CTFNO is based on Green’s formula of time-dependent PDEs \eqref{eq:time_Green}.
Here, we examine how useful the weight modulating structure of CTFNO is for learning time-dependent PDEs.
Rather than \textbf{weights}, there are three more places in the FNO network where temporal information can be mounted. We study these three alternative ways as follows:
\begin{itemize}[leftmargin=0.21in]
\item \textit{Baseline 1}: concatenate $t$ with an \textbf{input} function $u_0\left(x\right)$.
\item \textit{Baseline 2}: concatenate encoded time $\varphi_0(t)$ with a \textbf{lifted input} function $v_0(x)$.
\end{itemize}

\vspace{-17pt}
\begin{itemize}[leftmargin=0.21in]
\item \textit{Baseline 3}: concatenate encoded times $\varphi_0 (t)$, $\ldots$, $\varphi_L (t)$
with intermediate \textbf{features} $v_0(x), \ldots, v_L(x)$, respectively.
\end{itemize}
See Appendix \ref{appen: network_detail} for more details. 
We test the ability of these models to learn heat and Burgers' equations.
Results in Table \ref{tab:ablation} show that, in both examples, 
the ablation models result in significantly lower performance than CTFNO.
A method of concatenating time into the input allows us a simple way to put time information into the network, but the results confirm that it is not effective at all.
Despite the structure being capable of handling arbitrary times consecutively, the three ablation study models perform similarly or worse than FNO-2D, which can only evaluate solution values over time on a specific grid.
The results validate that the time modulating structure of CTFNO designed based on the time-dependent Green's formula is much more expressive in learning the PDE.

\subsection{Capability to represent diverse dynamics} \label{sec:synthetic}
In this section, we further harness the proposed CTFNO for modeling a variety of time-evolving dynamics, not confining to physical PDE problems. In all experiments, every model was trained with MSE loss.
\vspace{-12pt}
\paragraph{Why do we consider PDEs for time-series modeling?}
Starting with Neural ODEs, leveraging DEs has been found to be effective in modeling time-series data.
They have shown promising results, however, their model architectures and inference schemes are specialized to the specific DEs on which they are based.
These bespoke model structures rule out their generalization ability
to other classes of DEs, which is further exacerbated in real-world applications.
The necessity for a model capable of learning a wide range of dynamics has also been discussed by \citet{holt2022neural}.
However, these existing studies consider the target state only as a function of a single time variable, not a multivariate function of other variables as well as time. This makes the models difficult to understand which variables the dynamics depend on and how they relate to each other. It will be even more limited in real applications that approximate dynamics in latent space where how states evolve and what kinds of differential equations they follow are unknown.
On the other hand, time-dependent PDEs describe the evolution of a physical quantity, not only with time but also according to other variables such as spatial variables.
Due to their heavy expressivity,
PDEs are widely used to describe complex continuous processes \cite{temam2001navier,kulov2014mathematical,joshi2002optimal}.
In what follows, we show that our PDE-based model can better represent a diverse class of dynamics than existing DE-based models. 

\paragraph{Datasets} We use several illustrative examples to demonstrate the outstanding capacity of the proposed method in learning diverse classes of dynamics. A mathematical formulation of these dynamics can be found in Appendix \ref{sec:data}.
\begin{itemize}[topsep=0pt]
\item \textbf{Square and Sawtooth} \cite{bilovs2021neural} generate piecewise differentiable trajectories having cusps. We consider these to evaluate the capability on modeling waveform signals.
\item \textbf{Stiff ODE} \cite{holt2022neural} is a second-order ODE which exhibits regions of high stiffness. This is a typical example that Neural ODEs fail to learn.
\item \textbf{Spiral ODE}  \cite{bilovs2021neural} is a two-dimensional system of nonlinear ODEs, commonly arising in biological systems. The dynamics describe spiral shaped trajectories.
\item \textbf{Reaction ODE} is commonly used to model chemical reactions and is of the form 
$\partial u / \partial t=6u\left(1-u\right)$.
Unlike to aforementioned ODEs, a solution to Reaction ODE is regarded as a function. 
\end{itemize}
\paragraph{Baselines}
We evaluate the performance of CTFNO in comparison with several DE-based continuous time models: the standard Neural ODE (NODE; \citealt{chen2018neural}), ANODE \cite{dupont2019augmented}, and Neural Flow (NF; \citealt{bilovs2021neural}), which directly parametrizes the solution operator of an ODE, are adopted for ODE-based models.
We also consider Neural Laplace (NL; \citealt{holt2022neural}) which can represent diverse classes of equations by modeling them in the Laplace domain.

\begin{figure*}[t]
  \begin{center}
    \includegraphics[width=1.0\textwidth]{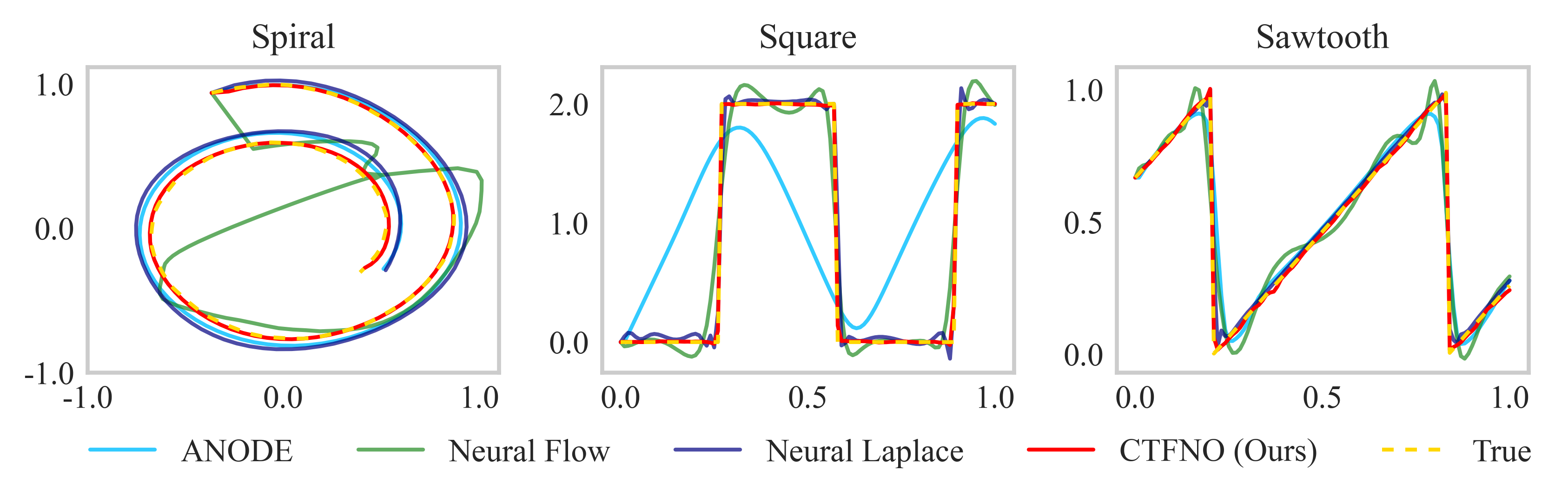}
  \end{center}
  \vspace{-15pt}  \caption{Visualization of learned solutions to synthetic datasets, which shows the superiority of our CTFNO.}
  \vspace{-10pt}
  \label{fig:synthetic}
\end{figure*}

\paragraph{Results}
The overall results of RMSE are reported in \cref{tab:synthetic} and \Cref{fig:synthetic} provides qualitative results of the learned solutions. 
Note that all models have a comparable number of parameters (See \cref{tab:synthetic_nparams}).
We can see that the performance of NODE and ANODE on different datasets varies a lot. They perform well on certain data and fail to learn the dynamics of another one.
Unlike these two, NF, NL, and our model directly learn the solution operator without using the numerical ODE solver.
Neural Flow parameterizes the ODE solution operator in the time domain. On the other hand, NL can describe more diverse dynamics by modeling them in the Laplace domain instead of the time domain. The results, which show the superiority of NL over NF, confirm that the range of dynamics that the model can describe is crucial.
For one-dimensional problems, CTFNO is similar but slightly better than NL. Qualitative comparisons in \cref{fig:synthetic} show that CTFNO approximates the discontinuity more accurately than NL without spurious oscillations. The advantage of CTFNO is evident in the spiral and reaction ODEs. 
Spiral is a system of ODEs, and reaction ODE describes the time-dependent evolution of a function defined in a spatial domain.
NL can cover a wide range of dynamics, however, it only considers univariate DEs that depend on only a single time variable. The same goes for other models.
On the other hand, our model can represent time-varying dynamics by considering other correlations. In \cref{tab:synthetic}, RMSE of CTFNO on reaction ODE is more than ten times better than other models and more than two times better for the spiral.
It confirms that the structure of CTFNO is helpful in approximating systems or high-dimensional dynamics.
In addition, the results in \cref{tab:NODE for PDE} show that PDE surrogates learn PDE solution operators much better than NF and NL (both cannot represent spatial relations). 
Moreover, results that show the superiority of CTFNO in extrapolating the spiral trajectories are provided in Figure \ref{fig:spiral_extrap}.
The overall results demonstrate how important the range of expressible dynamics of the model is, validating the suitability of CTFNO for learning a wide array of dynamics. 

\begin{table}[t]
    \centering
      \setlength\tabcolsep{20.5pt}
    \caption{RMSE ($\times10^{-2}$) results on synthetic data.} \label{tab:synthetic}
    \vspace{5pt}
     \scalebox{1.0}{
    \begin{tabular}{cccccc}
    \toprule
    Model & Square & Sawtooth & Stiff & Spiral & Reaction  \\
    \midrule
    NODE      & 97.50  & 28.09 &  44.23    & 3.26  & 5.109  \\
    ANODE       & 80.95  & 9.92 & 37.18    & 3.26    & 4.241   \\
    Neural Flow    &  20.40  & 7.93  &26.39 & 3.26    & 10.300  \\
    Neural Laplace & 17.06  & 4.84  & 20.83 & 4.25    & 2.804  \\
    \rowcolor{lightpurple}
    CTFNO         & \textbf{11.69}  & \textbf{3.90}  &  \textbf{16.53}  & \textbf{1.87}    & \textbf{0.239} \\
    \bottomrule
  \end{tabular}}
    \vspace{-5pt}
\end{table} 

\subsection{Real Time-Series Applications} \label{sec:real}
This section is devoted to investigating the performance of our model on interpolation and prediction tasks on real-world time-series datasets, including partially observed, multi-variate sequences.
\paragraph{Time-Series Modeling in Latent Space}
To represent sporadically observed real time-series data, we follow the encoder-decoder framework of Latent ODEs \cite{rubanova2019latent} that leverage a VAE \cite{kingma2013auto, szegedy2013intriguing} architecture to represent incomplete time-series data as a continuous-time model.
To focus on the model representation of inherent dynamics in the latent space, we employ a simple RNN encoder.
By passing a given input time-series through the RNN encoder, a latent vector $z_0$ is sampled using the last feature vector as the mean $\mu$ and standard deviation $\sigma$, which is the same as VAE.
Now, assuming implicit dynamics in the latent space with $z_0$ as the initial point, CTFNO returns an output vector at the desired time steps.
Finally, after passing through a decoder with a fully connected layer, the difference between the output and the target becomes the loss function for training.
Precise implementation details can be found in Appendix \ref{appen:test_detail}.

\paragraph{Datasets}
We evaluate our model on three real time-series data. For more details, see Appendix \ref{sec:data}.
\begin{itemize}[topsep=0pt]
\item \textbf{MuJoCo} \cite{tassa2018deepmind}
hopper environment from the Deepmind Control Suite records
14-dimensional attributes, including state and action, with 100 timestamps.
To deal with partial observations, we conduct interpolation and prediction tasks, in which we reveal either 10\%, 20\%, 30\%, or 50\% of the ground truth.
\vspace{3pt}
\item \textbf{PhysioNet 2012} \cite{silva2012predicting}
is an irregularly sampled real-world clinical dataset, which is investigated to evaluate our model on sparsely observed time-series.
The goal is to interpolate and predict 41 biomedical features, such as heart rate and glucose, of intensive care unit (ICU) patients. 
\vspace{3pt}
\item \textbf{Human Activity} \cite{kaluvza2010agent}
consists of sporadically observed sensor data collected from five individuals performing several activities (i.e. walking, standing, etc).
We use pre-processing steps as they were provided by \citet{rubanova2019latent}, resulting in 6554 sequences of 211 time points.
We train the models to classify the type of human activities from sequential data.
\end{itemize}
\paragraph{Baselines} DE-based models applied to real time-series data are chosen for the comparisons:
RNN-VAE is a variational autoencoder (VAE; \citealt{kingma2013auto, rezende2014stochastic}) model whose encoder and decoder are recurrent neural networks (RNNs).
ODE-RNN \cite{rubanova2019latent} is a RNN model which uses Neural ODEs to model hidden state dynamics.
Two Latent ODE (LODE) models with RNN \cite{chen2018neural} and ODE-RNN \cite{rubanova2019latent} encoders are also considered.
Finally, Coupling Flow in Neural Flow \cite{bilovs2021neural}, an ODE solution operator which directly models the solution curves of an ODE, is compared.
\begin{table}[t]
  \caption{Interpolation and prediction MSE ($\times 10^{-3}$) on the MuJoCo dataset.}
  \vspace{5pt}
  \label{tab:mujoco}
  \centering
  \setlength\tabcolsep{15.8pt}
  \scalebox{0.90}{
  \begin{tabular}{ccccccccc}
    \toprule
    Model & \multicolumn{4}{c}{Interpolation (\% Observed Points)} & \multicolumn{4}{c}{Prediction (\% Observed Points)} \\
    \cmidrule(lr){2-5} \cmidrule(lr){6-9}
     &\multicolumn{1}{c}{$10\%$} &\multicolumn{1}{c}{$20\%$} & \multicolumn{1}{c}{$30\%$} & \multicolumn{1}{c}{$50\%$} &\multicolumn{1}{c}{$10\%$} &\multicolumn{1}{c}{$20\%$} & \multicolumn{1}{c}{$30\%$} & \multicolumn{1}{c}{$50\%$}\\
    \midrule
    RNN-VAE & 65.14 & 64.08 & 63.05 & 61.00& 23.78 & 21.35 & 20.21 & 17.82 \\
    ODE-RNN & 16.47 & 12.09 & 9.86 & 6.65 & 135.08 & 319.5 & 154.65 & 264.63 \\
    LODE \scriptsize{(RNN enc)} & 24.77 & 5.78 & 27.68 & 4.47 & 16.63 & 16.53 & 14.85 & 13.77 \\
    LODE \scriptsize{(ODE enc)} & 3.60 & 2.95 & 3.00 & 2.85 & 14.41 & 14.00 & 11.75 & 12.58 \\
    Neural Flow &  7.15 & 5.58 & 4.96 & 4.60 & 17.99 & 16.10 & 15.48 & 15.29\\
    \rowcolor{lightpurple}
    CTFNO & \textbf{1.53} & \textbf{1.18} & \textbf{1.12} & \textbf{1.15} & \textbf{9.26} & \textbf{8.93} & \textbf{8.42} & \textbf{8.70} \\
    \bottomrule
  \end{tabular}}
\end{table}

\paragraph{Results} 
\begin{wraptable}{r}{0.47\textwidth}
    \centering
    \vspace{-20pt}
    \setlength\tabcolsep{6.3pt}
    \caption{MSE ($\times 10^{-3}$) on PhysioNet and per-time-point classification accuracies ($\%$) on Human Activity.} \label{tab:physio_activity}
     \vspace{4pt}
     \scalebox{0.95}{
  \begin{tabular}{cccc}
    \toprule
    Model & \multicolumn{2}{c}{PhysioNet} & \multicolumn{1}{c}{Activity} \\
    \cmidrule(lr){2-3} \cmidrule(lr){4-4}
     &Interpolation & Prediction & Accuracy\\
    \midrule
    RNN-VAE & 5.93 & 3.05 & 34.3  \\
    ODE-RNN & 2.36 & - & 82.9 \\
    LODE \scriptsize{(RNN enc)} & 3.16 & 5.78 & 83.5 \\
    LODE \scriptsize{(ODE enc)} & 2.23& 2.95 & 84.6 \\
    Neural Flow & 2.94 & - & 65.7\\
    \rowcolor{lightpurple}
    CTFNO & \textbf{1.80} & \textbf{2.10} & \textbf{85.0}  \\
    \bottomrule
    \end{tabular}}
    \vspace{-0pt}
    \end{wraptable}
In all experiments, we build models with comparable network sizes to make a fair comparison. 
Results on MuJoCo reported in \cref{tab:mujoco} show that CTFNO consistently outperforms the baseline models to a large extent on all tasks of both interpolation and prediction across all kinds of observed time points.
Another observation is that the performance difference between the cases with many and few observation points is the smallest.
These imply that our model approximates the time evolution of latent variables well, even with a small number of observations.
The results in \cref{tab:mujoco} suggest that CTFNO can be leveraged as a relevant model for real applications on time-series with missing time steps.
\cref{tab:physio_activity} summarizes interpolative and predictive performance on PhysioNet and per-time-point classification accuracy on Activity.
CTFNO is superior to all of the benchmark models on both PhysioNet and Activity, which indicate that our model provides a useful utilization of sparsely observed time-series data and a meaningful representation for classification.
Moreover, the numerical integration used in ODE-based models is computationally expensive and sometimes shows numerical instability (See \cref{tab:Time_elapsed}).
On the other hand, Neural Flow and our method efficiently predict the latent trajectories without the need for costly numerical schemes.
The overall results demonstrate that our proposed model consistently improves the performance of baseline models in real applications and is a novel approach that can handle a wide range of real-world problems.

\subsection{Stability and Generalization} \label{sec:exp_stab}
\paragraph{Data Efficiency}
\begin{wrapfigure}{r}{0.4\textwidth}
\begin{center}
  \vspace{-20pt}  
  \includegraphics[width=0.4\textwidth]{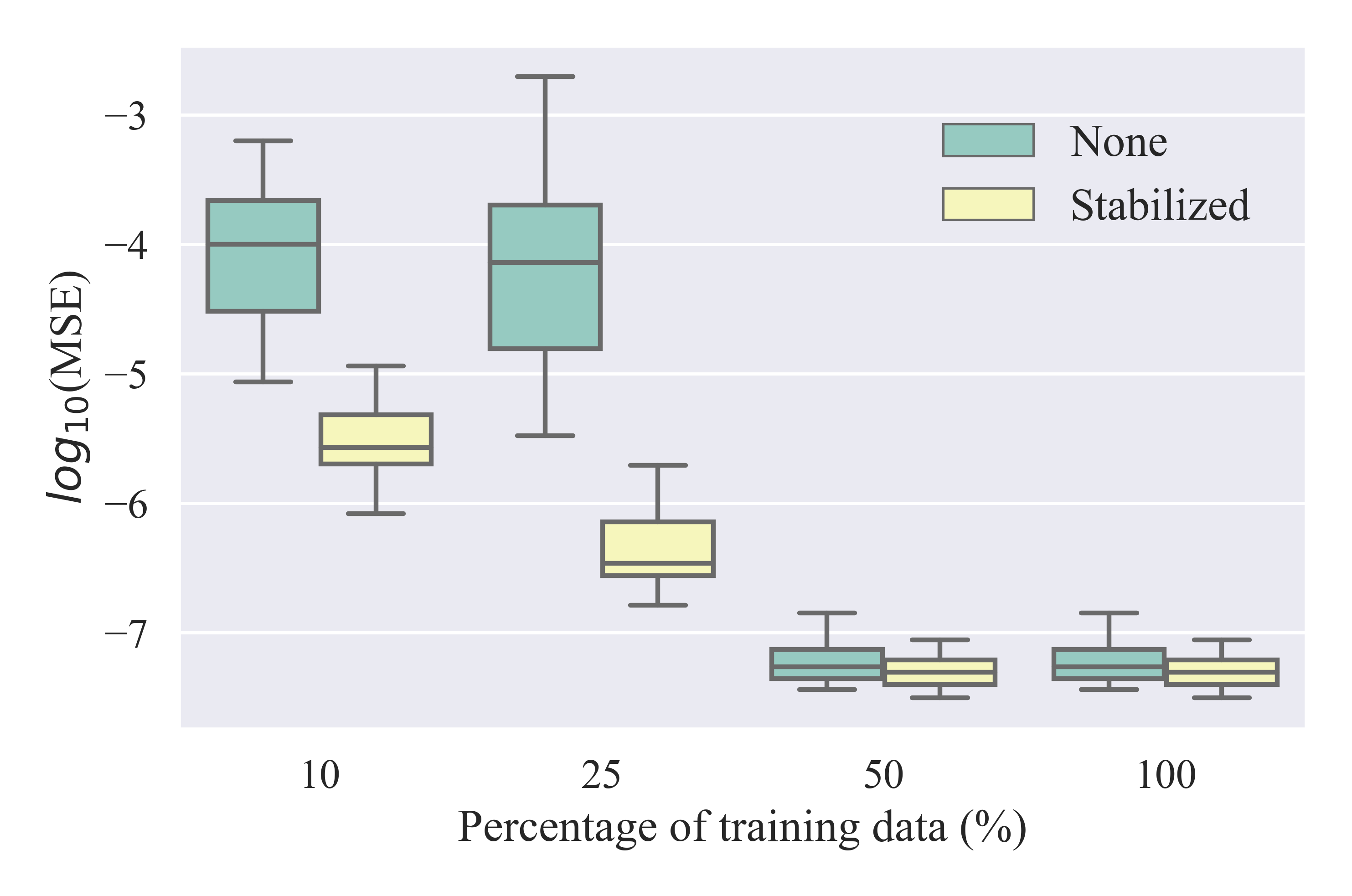}
  \end{center}
  \vspace{-12pt}
  \caption{Evaluation loss of CTFNO on heat equation. Stabilization improves the data efficiency.}
  \label{fig:data_efficiency}
   \vspace{-5pt}
\end{wrapfigure}
 
One of the main drawbacks of neural operators is that they require a wealth of available data.
A large corpus of input-output pairs of numerical PDE solutions is costly to generate. 
Even if the data is obtained from observations of the physical phenomena, they could be scarce.
Therefore, successfully training neural operators with small data is very important to make them useful in real applications.
In \cref{fig:data_efficiency}, we investigate that the stabilization scheme proposed in \cref{sec:Stab} improves the data efficiency capability. Boxplots report test MSEs of CTFNOs with and without stabilization per varying percentages of training data.
We can see that the error of the non-stabilized CTFNO increases a lot when the training data is of limited quantity.
Meanwhile, the stabilized CTFNO retains a consistently smaller test error with lower variance in all scenarios. 
These show that the stabilization scheme allows the model to make better use of the data, resulting in efficient learning without a lot of expensive data. 
Moreover, since the training cost is proportional to the size of the training dataset, stabilization provides an effective way to reduce the training cost.
The stabilization that can achieve these benefits brings our PDE surrogates one step closer to practical applications.

\paragraph{Model Generalization}
 \begin{wrapfigure}{r}{0.40\textwidth}
    \vspace{2pt}
    \centering
    \includegraphics[width=0.38\textwidth]{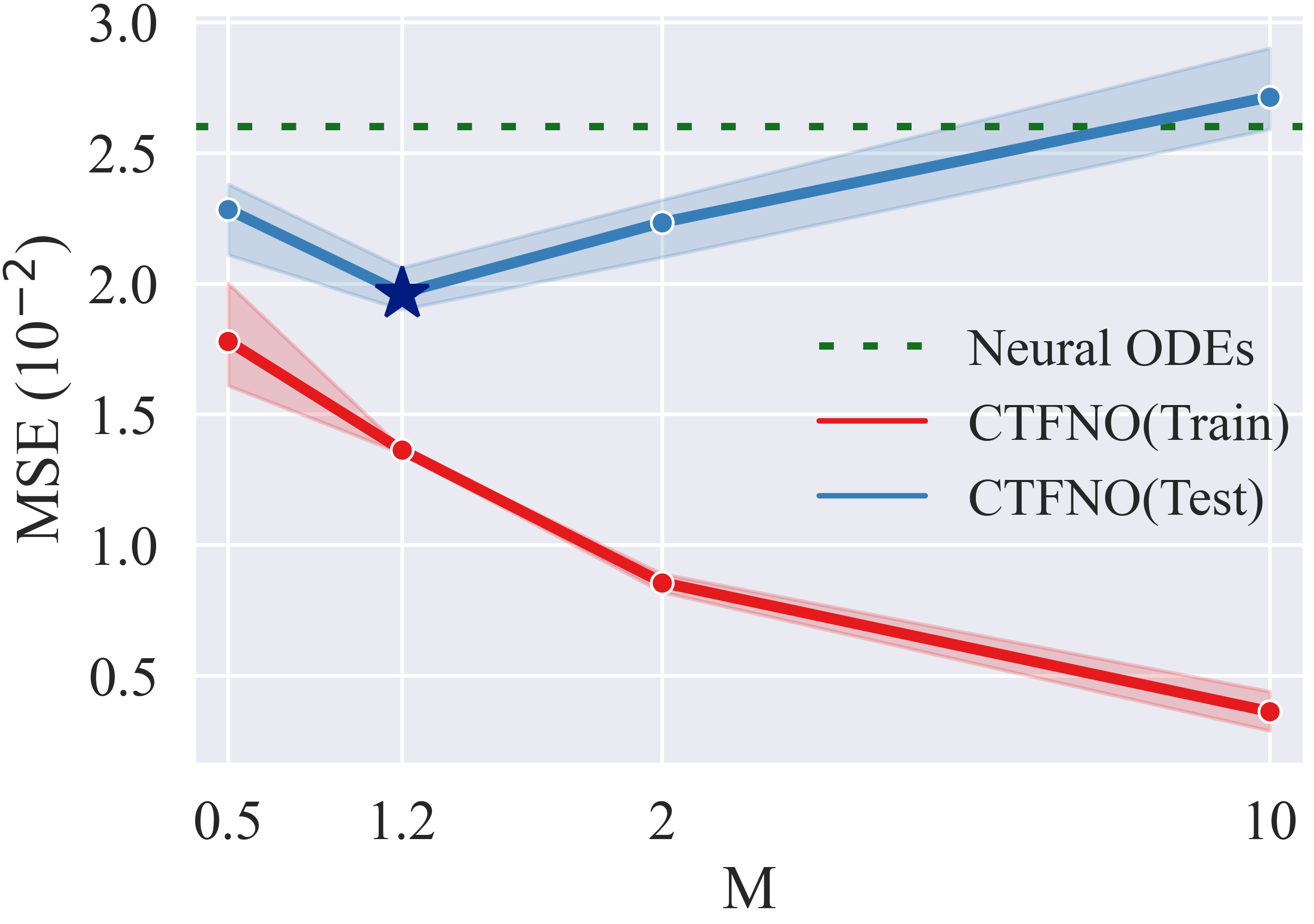}
   \vspace{-2pt}
      \caption{Effect of stabilization on the model generalization, tested on the Plane Vibration dataset.} \label{fig:pv}
  \vspace{-5pt}
  \end{wrapfigure}
Plane Vibration (PV; \citealt{noel2017f}) is a multi-variate data, consisting of five features; force, voltage, and accelerations measured in three spots.
We comply with the experimental setup as provided in HBNODEs \cite{xia2021heavy}, whose
task is to forecast the next eight time steps from the previous 64 consecutive time observations.
PV dataset is considered to elucidate how the stabilization scheme works in practical time-series data. 
MSE ($\times10^{-2}$) of second-order Neural ODEs, containing SONODE \cite{norcliffe2020second} and HBNODE \cite{xia2021heavy}, reported in \cite{xia2021heavy} is $\geq$ 2.5, and it is annotated by a green dashed line in \cref{fig:pv}.
Our CTFNO achieves much lower MSE of 1.96.
\Cref{fig:pv} also presents a positive effect of stabilization; a tendency of CTFNO to be less prone to overfitting.
We investigate the variability of the performance of CTFNO with respect to the Gershgorin's disc stabilization parameter $M$.
The ability to generalize well outside the training dataset is essential for models to be practically useful. We observe that increasing the $M$ leads to overfitting, and too small $M$ produces a degenerated system, dropping the performance.
Besides, the model with congenial stabilization does not suffer from overfitting or degradation of accuracy, which validates the effect of stabilization. Moreover, it is notable that Neural ODEs require tremendous computational burdens due to their use of numerical ODE solvers.
By obviating the need for the ODE solver, the computational time of CTFNO is merely about 10\% of Neural ODEs.

\begin{remark}
Our stabilization controls the amplification of output changes in response to input perturbations.
In addition to the results discussed here, there are more effects of the stabilization; it prevents overfitting and enhances the robustness against noisy observations and adversarial attacks.
We refer to \cref{appen:classification,appen:heatmaps_ode_heat_burgers} for further investigation of the effects of stabilization over model generalization.
\end{remark}

\section{Related Work}
\paragraph{Neural PDE Surrogates}
Pioneering works \cite{raissi2019physics, sirignano2018dgm} have incorporated physical principles into neural networks directly into loss functional. But, they have a challenging optimization landscape \cite{wang2021understanding, krishnapriyan2021characterizing} and require training a new network for a new PDE instance.
An orthogonal class of methods is autoregressive approaches \cite{brandstetter2022message, horie2022physics}, which solve PDEs iteratively. They have been well-suited to irregular boundaries and integrated with existing numerical PDE schemes and benefit from that. But, repetitively applying rollouts, even using numerical ODE solvers \cite{bar2019learning, lienen2022learning}, requires high computational costs and often makes the model hard to train.
Recently, a new line of work on operator learning  \cite{kovachki2021neural} has learnt a mapping from initial/boundary conditions to solutions. Some \cite{lu2019deeponet} learns a basis expansion of operators, others \cite{li2020neural, gin2021deepgreen, you2022nonlocal, salvi2021neural} use a neural network as the ansatz of the solution integral operator.
In this work, we focus on Fourier neural operator (FNO; \citealt{li2020fourier}), which has delivered success in learning various PDEs and vision representation \cite{guibas2021adaptive}. 
\vspace{-5pt}
\paragraph{Dynamics-based Time-Series Models}
The interpretation of a residual network \cite{he2016deep} as a discretization of an ODE \cite{chen2018neural, haber2017stable, lu2018beyond} provided an interface between deep learning and ODEs. 
Subsequently, extensive work has been conducted on parametrizing the continuous dynamics of hidden states using an ODE \cite{greydanus2019hamiltonian, lu2019understanding, liu2021second}.
Owing to the continuous representation of neural networks, Neural ODEs are particularly attractive for irregularly-sampled time-series data \cite{rubanova2019latent, de2019gru, chang2019antisymmetricrnn, kidger2020neural, chen2019symplectic}.
A recent work \cite{bilovs2021neural} circumvents the usage of an expensive numerical integration by directly parametrize the solution trajectory of an ODE.
However, they employ over simplistic ODEs, leading to constraints on the transformation of data, which limit expressivity of the models. To tackle this limitation, there have been several parallel attempts to introduce more diverse differential equation, including controlled DEs \cite{kidger2020neural}, delay DEs \cite{zhu2021neural, anumasa2021delay}, integro DEs \cite{zappala2022neural}, and Laplace transform-based method \cite{holt2022neural}. 
Also, some studies have integrated PDEs into the design of neural networks \cite{eliasof2021pde, ruthotto2020deep, ben2021quantized, sun2020neupde, kim2020pde}. But, all of these works assign author-defined specific PDEs to neural networks and all these are not applied to real time-series applications.

\section{Conclusion \& Limitations} \label{sec:conclusion}
In this paper, we presented a novel approach for modeling time-series in terms of PDEs.
As time is intrinsically continuous, we proposed a neural operator CTFNO that learns the underlying PDE by equipping FNO with the ability to represent time in a continuous manner.
We also provided theoretical guarantees for the universal approximation and the stability of CTFNO.
Our comprehensive experiments demonstrated that CTFNO outperforms existing differential equation-based models on synthetic and real-world datasets, and the proposed stabilization method effectively improves model generalization and robustness.

We note that FNO is hard to embody discontinuous features because Fourier transform only captures global information.
We expect that extending our approach to a model that can extract local features, such as \cite{gupta2021multiwavelet}, provides interesting avenues for future work.
Although stabilization increases the generalization of the model, the optimal $M$ for the given data is unknown, and strongly granted stabilization for some problems, such as ill-posed PDEs, may cause performance degradation.

\nocite{langley00}

\bibliography{mybib}
\bibliographystyle{./Template/Format}

\newpage
\onecolumn
\appendix
\appendix

\section{Universal Approximation} \label{appen:universal}
\subsection{Mathematical Notation} \label{appen:notation}
We introduce the symbols and mathematical notations that are frequently used in this paper.
\begin{center}

\begin{tabular}{l  p{0.7\textwidth}  r}
\textbf{Symbol} & \textbf{Description} 
\\
\midrule
$\sigma$ & activation function \\
$d$ & spatial dimension of domain \\
$\T^d$ & periodic torus, identified with $[0,2\pi]^d$ \\
$d_a$ & dimension of the input $a\left(x\right)$ \\
$d_u$ & dimension of the output solution $u\left(x\right)$ \\
$d_v$ & dimension of the augmented representation $v\left(x\right)$ \\
$T$ & terminal time \\
$L$ & the number of layers \\
$\Omega\subset\bR^d$ & spatial domain \\
$x$ & point in the the spatial domain \\
$\xi$ & frequency variable in the Fourier domain \\
$\cA(\Omega;\R^{d_a})$ & input function space consists of $a:\Omega\rightarrow\bR^{d_a}$ \\
$\cU(\Omega;\R^{d_u})$ & output function space consists of $u:\Omega\rightarrow\bR^{d_u}$\\
$\cF_X$, $\cF_{X}^{-1}$ & Fourier transform and inverse Fourier transform with respect to $X$ \\
$\cF_N$, $\cF_N^{-1}$ & discrete Fourier transform and inverse \\
$\cK_N$ & set of discrete Fourier wavenumbers $\cK_N = \set{\xi\in \Z^d}{|\xi|_\infty \le N}$ \\
$\cP$ & lifting operator \\
$\cL_\ell$ & neural operator layer \\
$\cQ$ & projection operator \\
$L^2$ & space of square-integrable functions \\
$L^2_N$ & $L^2_N\subset L^2$ trigonometric polynomials of degree $\le N$ \\
$H^s$ & Sobolev space of smoothness $s$, with norm $\Vert \slot \Vert_{H^s}$ \\
$P_N$ & $L^2$-orthogonal Fourier projection $P_N: L^2 \to L^2_N$ \\
$C\left(\left[0,T\right]\times X\right)$ & space of continuous functions $u:\left[0,T\right]\rightarrow X$ satisfying $\left\Vert u\right\Vert_{C\left(\left[0,T\right]\times X\right)}=\underset{0\leq t\leq T}{\max}\left\Vert u\left(t\right)\right\Vert < \infty$ \\
$\mathbf{D}_{n}$ & space of $n\times n$ diagonal matrices  \\
$I_n$ & $n\times n$ identity matrix \\
$\odot$ & component-wise matrix multiplication\\
\bottomrule
\end{tabular}

\end{center}

\subsection{Proof of universal approximation} \label{appen:universal_pf}
A neural operator
\begin{equation}
    \begin{aligned}
        \cN:\dA&\rightarrow\dCTU\\
        a&\mapsto\cN(a)
    \end{aligned}
\end{equation}
is defined as the following form 
\begin{equation}
\label{eq:tfno}
\cN(a)=\cQ\circ\cL_L\circ\cdots\cL_1\circ\cP(a)
\end{equation}
with a lifting operator
\[
\begin{aligned}
    \cP:\dA&\rightarrow\dCTU[v]\\
    \cP(a)(t,x)&\coloneqq Pa(x),&P\in\bR^{d_v\times d_a}
\end{aligned}
\]
for all $a\in\dA$, $x\in \Omega$, and a projection operator
\begin{align*}
    \cQ:\dCTU[v]&\rightarrow\dCTU\\
    \cQ(v)(t,x)&\coloneqq Qv(t,x),&Q\in\bR^{d_u\times d_v}.
\end{align*}
Moreover, as introduced in Section \ref{sec:tfno} the time-dependent neural operator layer is defined by

\begin{equation}
    \begin{aligned}
    &\cL_\ell(v)(t,x)\\
    &\coloneqq\sigma\left(W_\ell(t)v(t,x)+b_\ell(t,x)+\left(\cK\left(a;\theta_\ell\right)v\right)(t,x)\right)\\ 
    &\coloneqq\sigma\left(W_\ell(t) v(t,x)+b_\ell(t,x)+\cF_X^{-1}\left(R_\ell\left(t,\xi\right)\cdot\cF_X\left(v\right)\left(\xi\right)\right)\left(x\right)\right)\\
    &\coloneqq\sigma\left(W_\ell\psi_{w,\ell}(t) v(t,x)+ \psi_{b,\ell}(t)b_\ell(x)+
    \cF_X^{-1}\left(\varphi_{\ell}(t,\xi) R_{\ell}(\xi)\cdot\cF_X(t,v)(\xi)\right)(t,x)
    \right).
    \end{aligned}
\end{equation}

Before we establish a proof of the universality of CTFNO, let us start by recalling the following Lemma provided in \citep{kovachki2021universal}.
\begin{lemma}[\citep{kovachki2021universal}] \label{lem:fno}
For FNO introduced in Section \ref{sec:fno}, the following holds.
\begin{enumerate}[label=(\alph*)]
    \item Let $\cG: H^{s}\left(\bT^d;\bR^{d_a}\right) \rightarrow H^{s'}\left(\bT^d;\bR^{d_u}\right)$ be continuous operator with $s,s'\ge 0$ and $K\subset H^s$ be a compact subset. Then for any $\eps>0$ there exists $N\in \bN$, such that for all $a\in K$
    \[
    \norm[L^2]{\cG(a)-\cG_N(a)}<\eps,
    \]
    where $\cG_N:H^{s}\left(\T^{d}\right)\rightarrow L^{2}\left(\T^{d}\right)$ is defined by $\cG_N\left(a\right)\coloneqq P_N\cG\left(P_N a\right)$.
    \item Let $B>0$ and $\eps>0$ be given. Then there exists a FNO $\cN_{FT}$ such that for all $v\in L^2$ with $\norm[L^2]{v}\le B$
    \[
    \norm[\infty]{\cF_N\circ P_Nv-\cN_{FT}v}<\eps.
    \]
    \item Let $B>0$ and $\eps>0$ be given. Then there eixsts a FNO $\cN_{IFT}$ such that for all $\{\hat{v}(\xi):\xi\in\cK_N\}$ with $\left\|\hat{v}\right\|_{L^\infty(\cK_N)}\le B$
    \[
    \norm[L^2]{\cF^{-1}_N\left(\left\{\hat{v}(\xi)\right\}\right)-\cN_{IFT}\left(\left\{\hat{v}(\xi)\right\}\right)}<\eps.
    \]
    \item Let $B>0$ and $\eps>0$ be given. Then there exists a FNO $\widehat{\cN}$ such that for any $\left\{\hat{v}(\xi)\right\}\in\bC^{\cK_N}$ with $\left\|\hat{v}\right\|_{L^\infty(\cK_N)}\le B$ such that
    \[
    \norm[\infty]{\widehat{\cG}_N\left(\left\{\hat{v}(\xi)\right\}\right)-\widehat{\cN}\left(\left\{\hat{v}(\xi)\right\}\right)}<\eps.
    \]
\end{enumerate}
\end{lemma}

In what follows, we prove our main theorem showing that the proposed CTFNO is a universal approximator of continuous operators. The formal statement of the \cref{thm:universal} is given as follows:
\begin{theorem}[Universal approximation of CTFNO]
\label{thm:universal-tfno}
Let $s,s'\ge0$. Suppose $\cG:\dH{\left(\bT^d;\bR^{d_a}\right)}\rightarrow\dCTH[s']{\left(\bT^d;\bR^{d_u} \right)}$ is a continuous operator and $K\subset \dH{\left(\bT^d;\bR^{d_a}\right)}$ is a compact subset. Then for any $\epsilon>0$, there exists a CTFNO $\cN:\dH{\left(\bT^d;\bR^{d_a}\right)}\rightarrow\dCTH[s']{\left(\bT^d;\bR^{d_u} \right)}$, of the form \eqref{eq:tfno}, continuous as an operator $\dH{\left(\bT^d;\bR^{d_a}\right)}\rightarrow C\left([0,T]\times H^{s'}\left(\bT^d;\bR^{d_u} \right)\right)$, such that
\[
\sup_{a\in K}\|\cG(a)-\cN(a)\|_{C\left([0,T]\times H^{s'}\left(\bT^d;\bR^{d_u} \right)\right)}<\epsilon.
\]
\end{theorem}

The proof is essentially split into two parts. The first part states that the Fourier projection can approximate continuous operators.
To simplify notation, we use shorthand
$H^{s}=\dH{\left(\bT^d;\bR^{d_a}\right)}$ , $H^{s'}=H^{s'}\left(\bT^d;\bR^{d_u} \right)$ and  $a_t(x)\coloneqq a(t,x)$ in the remainder of the paper.
\begin{lemma}\label{lem:lem1}
Let $\cG:\dCTH[s]{}\rightarrow \dCTH[s']{}$ be a continuous operator and $K\subset \dCTH[s]{}$ a compact subset. Then for any $\eps>0$ there exists $N\in \bN$, such that for all $a\in K$ and $t\in[0,T]$
    \[
    \norm[H^{s'}]{\cG(a)_t-\cG_N(a)_t}<\eps.
    \]
\end{lemma}
\begin{proof}
Since $\cG$ is uniformly continuous on a compact domain $\cG(K)$, for any $\eps>0$ there exists $\delta>0$ such that
\[
\norm[H^{s'}]{\cG(a)_{t_1}-\cG(a)_{t_2}}<\eps,\quad\forall a\in K,
\]
provided $\left|t_1-t_2\right|<\delta$.
Let $0=t_0<t_1<\ldots<t_M=T$ be a partition of $[0,T]$ with $t_{i+1}-t_i<\delta$
and $\cG_i:H^s\rightarrow H^{s'}$ be defined by $a\mapsto \cG(a)_{t_i}$ for each $i$. Then, (a) in Lemma \ref{lem:fno} furnishes that there exists $N_i\in\bN$ such that
\[
\norm[H^{s'}]{\cG(a)_{t_i}-\cG_{N_i}(a)_{t_i}}<\eps,\quad\forall a\in K.
\]
Combining all, for $N>\max\left\{N_i:i=0,\ldots,M\right\}$, we have 
\[
\begin{aligned}
    &\norm[H^{s'}]{\cG(a)_t-\cG_N(a)_t}\\
    &\quad\le\norm[H^{s'}]{\cG(a)_t-\cG(a)_{t_i}}+\norm[H^{s'}]{\cG(a)_{t_i}-\cG_N(a)_{t_i}}+\norm[H^{s'}]{\cG_{N}(a)_{t_i}-\cG_N(a)_t}\\
    &\quad<3\eps,
\end{aligned}
\]
for any $a\in K$ and $t\in[0,T]$. This proves the lemma.
\end{proof}

The second part is to prove that any time-dependent continuous function can be approximated by a three-layered feed forward network combined with the proposed time-module.
\begin{lemma} \label{lem:time_cybenko}
Let $f:[0,T]\times\bR^{d_x}\rightarrow \bR^{d_y}$ be a continuous function and $K\subseteq \bR^{d_x}$ a compact subset. Then for any $\epsilon>0$, there exists a three-layered feed forward network with time
\begin{equation*}
    \begin{aligned}
    \cN&\coloneqq \cQ\circ\cL_2\circ\cL_1,\\
    \cL_\ell(t,x)&\coloneqq\left(t,\sigma\left(\psi_{w,\ell}(t)W_\ell x+\psi_{b,\ell}(t)b_\ell\right)\right),\\
    \cQ(t,x)&=Qx,
    \end{aligned}
\end{equation*}
satisfying
\[
\norm[]{\cN(t,x)-f(t,x)}<\eps,\quad\forall (t,x)\in [0,T]\times K,
\]
where $\sigma$ is an activation function, $\psi_{w,\ell}:[0,T]\rightarrow \mathbf{D}_{d_v}$ and $\psi_{b,\ell}:[0,T]\rightarrow\bR^{d_v}$ are continuous, $Q$ is a projection matrix, and $W_\ell,b_\ell$ are weight matrix and bias vector for each $\ell$, respectively.
\end{lemma}
\begin{proof}
Without loss of generality, we may assume $d_y=1$. Construct $W_1\in\bR^{(d_x+1)\times d_x}$ and $b_1\in\bR^{d_x+1}$ as 
\[
W_1=\begin{pmatrix}\mathbf{I}_{d_x}\\0\end{pmatrix}, b_1=\begin{pmatrix}\mathbf{0}_{d_x}\\1\end{pmatrix},
\]
with $\psi_{w,1}(t)=\mathbf{0}_{d_x+1}$, $\psi_{b,1}(t)=t\mathbf{1}_{d_x+1}$. Then, $\cL_1(t,x)=\left(t,(x,t)^T\right)\in[0,T]\times\bR^{d_x+1}$. By the universal approximation theorem of \citet{cybenko1989approximation}, for any $\eps>0$ there exists $W_2\in\bR^{d_v\times(d_x+1)}$, $b_2\in\bR^{d_v}$, and $Q\in\bR^{1\times d_v}$ such that
\[
\left|Q\sigma\left(W_2\begin{pmatrix}x\\t\end{pmatrix}+b_2\right)-f(t,x)\right|<\eps, \quad\forall (t,x)\in[0,T]\times K.
\]
By setting $\psi_{w,2}=\mathbf{I}_{d_v}$, $\psi_{b,2}=\mathbf{1}_{d_v}$, we have the desired $\cN$.
\end{proof}

With these lemmas out of the way, we are ready to provide a proof of Theorem \ref{thm:universal}. 

For given continuous operator
$\cG:\dCTH[s]{}\rightarrow C\left([0,T]\times L^2\right){}$, compact subset $K\subset \dCTH[s]{}$ and small $\eps>0$, Lemma \ref{lem:lem1} gives an integer $N\in\bN$ such that
\[
\sup_{a\in K}\norm[L^2]{\cG(a)_t -\cG_N(a)_t}<\eps,
\]
for each time $t\in[0,T]$. 
Since $K$ is compact, there exists $B>0$ satisfying $\norm[L^2]{v_t}\le B$ for all $v\in K$ and $t\in[0,T]$.
For given small $\delta\in(0,1)$, let $\cN_{FT}$ be a CTFNO such that for all function $w$ with $\norm[L^2]{w}\le B$
\[
\norm[\infty]{\cF_N\circ P_Nw-\cN_{FT}w}<\delta.
\]
In fact, $\delta$ is chosen according to the Lipschitz constant of $\widehat{\cN}$, which will be constructed.
As claimed by $(c)$ in Lemma \ref{lem:fno}, there exists a CTFNO $\cN_{IFT}$ such that for all $\{\hat{w}(\xi):\xi\in\cK_N\}$ with $|\hat{w}(\xi)|\le \sup_{t\in[0,T]}\sup_{a\in K}\norm[\infty]{\cF_N\circ P_N(a_t)}+1 \le B'=C(N)\sup_{t\in[0,T]}\sup_{a\in K}\left\Vert a_t\right\Vert_2 +1$
\[
\norm[L^\infty]{\cF_N^{-1}\left(\{\hat{w}(\xi)\}\right)-\cN_{IFT}\left(\{\hat{w}(\xi)\}\right)}<\eps.
\]
The continuity of $\cF_N^{-1}$ and $\cN_{IFT}$ leads to the existence of $B''>0$ that satisfies 
\[
\norm[\infty]{\cF_N^{-1}\left(\{\hat{w}(\xi)\}\right)}\le B'',\quad \norm[\infty]{\cN_{IFT}\left(\{\hat{w}(\xi)\}\right)}\le B'',
\]
for all $\{\hat{w}(\xi):\xi\in\cK_N\}$ with $|\hat{w}(\xi)|\le B''.$
By (d) in Lemma \ref{lem:fno} and Lemma \ref{lem:time_cybenko}, we can construct a CTFNO $\widehat{\cN}$ such that for any $\{\hat{w}(\xi)\}\in\bC^{\cK_N}$ with $\left\|\hat{w}\right\|_{L^\infty}\le B''$,
\[
\sup_{t\in\left[0,T\right]}\norm[\infty]{\widehat{\cG}_N\left(t,\{\hat{w}(\xi)\}\right)-\widehat{\cN}\left(t,\{\hat{w}(\xi)\}\right)}<\eps.
\]
Note that the construction of $\widehat{\cN}$ is independent on $\delta$ and $\cN_{FT}$. Now we set $\delta >0$ so small that
\[
\operatorname{Lip}\left(\widehat{\cN}\right)\sup_{t\in\left[0,T\right]}\norm[\infty]{\cF_N\circ P_N(a_t)-\cN_{FT}(a_t)}<\eps,
\]
for all $a\in K$, which implies that
\[
\sup_{a\in K}\norm[L^2]{\widehat{\cN}\circ\cF_N\circ P_N(a_t)-\widehat{\cN}\circ\cN_{FT}\left(a_t\right)}<\eps.
\]
Finally, for each $t\in\left[0,T\right]$ we can deduce the following inequality:
\begin{equation}
    \begin{aligned}
    &\sup_{a\in K}\norm[L^2]{\cG(a)_t-\cN(a)_t}\\
    &\quad<\eps+\sup_{a\in K}\norm[L^2]{\cG_N(a)_t - \cN(a)_t}\\
    &\quad= \eps+\sup_{a\in K}\norm[L^2]{\cF^{-1}_N\circ\widehat{\cG}_N\circ\cF_N\circ P_N\left(a_t\right) - \cN_{IFT}\circ\widehat{\cN}\circ\cN_{FT}\left(a_t\right)}\\
    &\quad\le\eps+\sup_{a\in K}\norm[L^2]{\cF^{-1}_N\circ\widehat{\cG}_N\circ\cF_N\circ P_N\left(a_t\right) - \cF_N^{-1}\circ\widehat{\cN}\circ\cN_{FT}\left(a_t\right)}\\
    &\qquad\quad+\sup_{a\in K}\norm[L^2]{\cF_N^{-1}\circ\widehat{\cN}\circ\cN_{FT}\left(a_t\right) - \cN_{IFT}\circ\widehat{\cN}\circ\cN_{FT}\left(a_t\right)}\\
    &\quad\le\eps+\sup_{a\in K}\norm[\infty]{\widehat{\cG}_N\circ\cF_N\circ P_N\left(a_t\right) - \widehat{\cN}\circ\cN_{FT}\left(a_t\right) }\\
    &\qquad\quad +\sup_{\substack{\{\hat{w}(\xi)\}\in\bC^{\cK_N}\\
    \left\|\hat{w}(\xi)\right\|_\infty\le B''}}
    \norm[L^2]{\cF_N^{-1}\left(\{\hat{w}(\xi)\}\right)-\cN_{IFT}\left(\{\hat{w}(\xi)\}\right)}\\
    &\quad <2\eps+\sup_{a\in K}\norm[\infty]{\widehat{\cG}_N\circ\cF_N\circ P_N\left(a_t\right)-\widehat{\cN}\circ\cF_N\circ P_N\left(a_t\right)}\\
    &\qquad\quad+\sup_{a\in K}\norm[L^2]{\widehat{\cN}\circ\cF_N\circ P_N\left(a_t\right)-\widehat{\cN}\circ\cN_{FT}\left(a_t\right)}\\
    &\quad\le2\eps+\sup_{\substack{\{\hat{w}(\xi)\}\in\bC^{\cK_N}\\
    \left\|\hat{w}(\xi)\right\|_\infty\le B'}}\norm[\infty]{\widehat{\cG}_N\left(\{\hat{w}(\xi)\}\right)-\widehat{\cN}\left(\{\hat{w}(\xi)\}\right)}\\
    &\qquad\quad+\sup_{a\in K}\norm[L^2]{\widehat{\cN}\circ\cF_N\circ P_N\left(a_t\right)-\widehat{\cN}\circ\cN_{FT}\left(a_t\right)}\\
    &\quad<3\eps+\sup_{a\in K}\norm[L^2]{\widehat{\cN}\circ\cF_N\circ P_N\left(a_t\right)-\widehat{\cN}\circ\cN_{FT}\left(a_t\right)}\\
    &\quad<4\eps.
    \end{aligned}
\end{equation}
Since $t\in\left[0,T\right]$ is arbitrary, the above inequalities bring us to the desired result
\[
\sup_{a\in K}\left\Vert\cG\left(a\right)-\cN\left(a\right)\right\Vert_{\dCTH[s']{}}<\eps.
\]

\section{Stability} \label{appen:stable}
\paragraph{Proof of stability}
Suppose $\sigma$ is a Lipschitz continuous activation function with Lipschitz constant $C$, and both $\sup_{\xi}\left\Vert R\left(t,\xi\right)\right\Vert _{2}^{2}$ and $\left\Vert W\left(t\right)\right\Vert _{2}^{2}$ are bounded by $M$ for every $t$. Let $\tilde{v}$ be a perturbed initial condition with $\left\Vert \tilde{v}-v\right\Vert _{2}<\eps.$
Then by Plancherel theorem and Fourier convolution theorem, we have

\begin{align*}
 & \left\Vert\cL_\ell\left(t,\tilde{v}\right)-\cL_\ell\left(t,v\right)\right\Vert_{L^2}^2  \\
 & = \left\Vert\sigma\left(W_\ell\left(t\right)\tilde{v}\left(x\right)+\left(\mathcal{K}_{\ell}\left(t\right)\ast \tilde{v}\right)\left(x\right)\right)-\sigma\left(W_\ell\left(t\right)v\left(x\right)+\left(\mathcal{K}_{\ell}\left(t\right)\ast v\right)\left(x\right)\right)\right\Vert_{L^2}^2\\ 
 &\leq C\left\Vert W_\ell\left(t\right)\left(\tilde{v}-v\right)\left(x\right)+ \left(\cK_\ell\left(t\right)\ast\left(\tilde{v}-v\right)\right)\left(x\right)\right\Vert_{L^2}^2 \\
 & =C\intop\left|W_\ell\left(t\right)\left(\tilde{v}-v\right)\left(x\right)+\left(\mathcal{K}_{\ell}\left(t\right)\ast\left(\tilde{v}-v\right)\right)\left(x\right)\right|^{2}dx\\
 & \leq C\intop\left|W_\ell\left(t\right)\left(\tilde{v}-v\right)\left(x\right)\right|^2dx+C\intop\left|R_\ell\left(t,\xi\right)\cdot\mathcal{F}\left(\tilde{v}-v\right)\left(\xi\right)\right|^{2}d\xi\\
 & \leq C\intop\left\Vert W_\ell\left(t\right)\right\Vert _{2}^{2}\left|\left(\tilde{v}-v\right)\left(x\right)\right|^{2}dx + C\intop\left\Vert R_\ell\left(t,\xi\right)\right\Vert _{2}^{2}\left|\mathcal{F}\left(\tilde{v}-v\right)\left(\xi\right)\right|^{2}d\xi\\
 & \leq CM\intop \left|\left(\tilde{v}-v\right)\left(x\right)\right|^{2}dx + CM\intop\left|
 \mathcal{F}\left(\tilde{v}-v\right)\left(\xi\right)\right|^{2}d\xi\\
 & =2CM\left\Vert \tilde{v}-v\right\Vert _{L^{2}}^{2}\\
 & <2CM\eps.
\end{align*}



\section{Experimental Details} \label{appen:test_detail}
In this section, we present our experimental settings in detail.
All experiments were conducted on a single NVIDIA RTX 3090 GPU.

\subsection{Datasets} \label{sec:data}
\paragraph{Heat equation}
The one-dimensional heat equation used in Section \ref{sec:synthetic} takes the form 
\begin{equation}
\begin{cases}
\frac{\partial u}{\partial t}=\nu \frac{\partial^2 u}{\partial x^2}, & x\in \left(0,1\right),\ t\in\left(0,2.5\right] \\
u\left(0,x\right)=u_{0}\left(x\right), &  x\in \left(0,1\right),
\end{cases}\label{eq:heat}
\end{equation}
with periodic boundary conditions. It describes how a quantity such as heat diffuses through a given region over time.
The initial function $u_0(x)$ is generated from the Gaussian random field $\mathcal{N}\left(0,20^2\left(-\triangle +3.5^2\right)^{-2.5}\right)$, where $\triangle$ refers to the Laplacian.
The diffusivity constant 
$\nu$ is set to be  0.001 and the solution data is generated by exactly solving (\ref{eq:heat}) in Fourier space on a uniform spatial grid with resolution 1024 along time step size $\triangle t=0.05$. 

\paragraph{Burgers' equation}
The one-dimensional Burgers' equation with diffusive regularization is a canonical non-linear PDE, taking the form
\begin{equation}
\begin{cases}
\frac{\partial u}{\partial t} + u\frac{\partial u}{\partial x}=\nu \frac{\partial^2 u}{\partial x^2}, & x\in \left(0,1\right),\ t\in\left(0,1\right] \\
u\left(0,x\right)=u_{0}\left(x\right), &  x\in \left(0,1\right),
\end{cases}\label{eq:burgers}
\end{equation}
with periodic boundary condition. 
It is a dissipative nonlinear system with shock formation and has various applications including the flow of viscous fluid dynamics. The viscosity is set to  $\nu = 0.001$.
Starting from initial functions sampled from $\mathcal{N}\left(0,7^2\left(-\triangle +49\right)^{-2.5}\right)$, numerical ground truth is generated separately where the linear diffusion component is exactly solved and the remaining nonlinear part is solved in Fourier space using forward Euler method as in \citep{li2020fourier}. The spatial domain is discretized with resolution 1024.
The training set consists of 400 periodic trajectories and each trajectory has regularly-sampled time points with $\triangle t=0.005$ and we test the trained models for 100 trajectories. 

\paragraph{Diffusion-Sorption equation}
Diffusion-sorption equation is an example taken from \textsc{PDEBench}.
It represents a diffusion process influenced by sorption interactions and is wirtten as
\begin{align}
\label{eq:diff-sorp}
    \partial_t u(t,x) & = D/R(u) \partial_{xx} u(t,x),  ~~~ x  \in (0,1),
\end{align}
where $D$ is the diffusion coefficient, $R(u)$ regulates constitutive relationship.
The spatial domain is discretized to 1024, and we train models to predict solution at 50 future times from the initial function.
We compose the training and test sets by 400 and 100 trajectories, respectively.

\paragraph{Compressible Navier-Stokes equations}
The motion of viscous fluid fields is described by the following compressible Navier-Stokes equations:
\begin{equation}
\begin{cases}
     \partial_t \rho + \nabla \cdot (\rho \textbf{v}) &= 0,
     \\
     \rho (\partial_t \textbf{v} + \textbf{v} \cdot \nabla \textbf{v}) &= - \nabla p + \eta \triangle \textbf{v} + (\zeta + \eta/3) \nabla (\nabla \cdot \textbf{v}),\\
     \partial_t \left[ \epsilon + \frac{\rho v^2}{2} \right] &+ \nabla \cdot \left[ \left(\epsilon + p + \frac{\rho v^2}{2} \right) \bf{v} - \bf{v} \cdot \sigma' \right] = 0,
 \end{cases}
\end{equation}
where $\rho$ denotes the density, $\textbf{v}$ the velocity field, $p$ the pressure, $\sigma '$ the viscous stress tensor, $\eta$ the shear viscosity, and $\zeta$ the bulk viscosity.
From \textsc{PDEBench} \cite{takamoto2022pdebench}, we take a one-dimensional example generated with periodic boundary conditions, and $\eta=\zeta=0.1$.
We use the data consisting of 1024 spatial resolutions and 50 time steps, 700 for training and 200 for validation.

\paragraph{Spiral} 
Data generation of Spiral, Stiff, Sawtooth, and Square follows the implementation of \citet{holt2022neural} unless stated.
We generate the Spiral dataset by following governing equation:
\begin{equation}
    \dot{u}(t) = A \tanh{{u}(t)},
\end{equation}
where $A=\begin{pmatrix} -1/8 & 1 \\ -1 & -1/8 \end{pmatrix}$.
The other settings such as the number of training and test samples, initial value $u(0)$, and the number of timesteps follow the implementation of \citet{holt2022neural}.

\paragraph{Stiff Van der Pol Oscillator}
We generate data by following the instructions of \citep{van1927frequency}, which exhibits regions of high stiffness. The governing equation is
\begin{align}
    \dot{x} &= y, \\
    \dot{y} &= \mu \left(1-x^2\right)y - x,    
\end{align}
where $\mu=1000$. We sample initial conditions from $x(0)\in [0.1, 2]$, $y(0)=0$.

\paragraph{Sawtooth \& Square}
We also explore the benchmarks on a periodic discontinuous function $u(t)$.
We sample two datasets, namely sawtooth and square, and these are sampled by the equation $u(t) = \frac{t}{2\pi} - \lfloor \frac{t}{2\pi} \rfloor$, and $u(t) = 2\left(1-\lfloor 2\left(\frac{t}{2\pi}- \lfloor \frac{t}{2\pi} \rfloor\right)\rfloor\right)$, respectively.
We sample initial values $\left(t_0, u(t_0)\right)$ by sampling $t_0$ in uniformly at random on the interval $[0,2\pi]$.
Each trajectories are generated from intervals of $[t_0, t_0+20]$.

\paragraph{Reaction equation}
The semi-linear ODE presented in Section \ref{sec:synthetic} has an analytic solution 
\begin{equation}\label{eq:sol_reaction}
u\left(t,x\right)=\frac{f\left(x\right)e^{\rho t}}{f\left(x\right)e^{\rho t}+1-f\left(x\right)},
\end{equation}
for an initial condition $f\left(x\right)$. The reaction coefficient $\rho$ is chosen to be 6 and the domain is unit interval. 
Initial conditions randomly generated as $\frac{1}{2}\left({z_1\sin\left(2\pi k_{1}x\right) + z_2\sin\left(2\pi k_{2}x\right)}\right)+z_3e^{-x}+2$, where $z_i\sim N\left(0,1\right)$ and $k_i$ is a uniformly sampled integer.
Analytic solutions are attained up to $t=1$. 
The solution data is generated by solving (\ref{eq:sol_reaction}). The spatial grid is discretized with resolution of 100 along time step size of $\Delta t = 0.02$.

\paragraph{MuJoCo}
We use a trajectory of physical simulation for the hopper with three joints and four body parts.
Each time series is 14-dimensional, consisting of a five-dimensional position, six-dimensional velocity, and three-dimensional action.
\citet{rubanova2019latent} generated $10{\small,}000$ sequences of 200 timesteps and used 80\% of the data for training and the rest 20\% for evaluation.
For the interpolation task, they randomly sampled 100 consecutive observation timesteps. 
For the prediction task, they use all 200 observation timesteps and divide them into two parts: the first 1/3 as an input for the model and the latter 2/3 as an output to be forecasted.
For all the tasks, they randomly sampled a portion of the input timesteps at a specified rate and masked the rest of them.
Note that the values used for the subsampling ratio are 10\%, 20\%, 30\%, and 50\%.

\paragraph{PhysioNet}
This is a publicly available dataset of 8000 time series describing the stay of patients within an ICU over 48 hours.
For each patient, 41 biomedical features are irregularly observed and converted to one minutely resolution. 
In \citep{rubanova2019latent}, they used 80\% of the data for training, and the rest 20\% for evaluation, as in MuJoCo.
For the interpolation task, they did not subsample the data because the measurements were already sparse.
For the prediction task, they halved the data so that the first 1/2 timesteps were used for inputs, and the latter 1/2 timesteps were used for outputs to be forecasted.
Note that \citet{rubanova2019latent} also performed per-sequence classification experiments on PhysioNet.
Since Latent ODEs with RNN or ODE-RNN encoders compress the whole input data to a single latent vector $z_0$ and an MLP classifier receives only $z_0$ as its input, such a per-sequence classification does not depend on all the other latent vectors $z_1, \ldots, z_T$.
This is inconsistent with the goal of CTFNO, which focuses on learning representations of the inherent dynamics for all timesteps.
Hence we excluded per-sequence classification experiments in this paper.

\paragraph{Activity}
This dataset consists of time series from five individuals performing serveral activities (i.e. walking, standing, laying, etc).
Each time series includes 12 features indicating tags attached to their belt, chest and ankles.
We used the same pre-processing as \citep{rubanova2019latent}, resulting in 6554 sequences of 211 time points.
The task is to classify each time point into one of seven types of activities and we performed per-time-point classification experiments as in \citep{rubanova2019latent}.

\paragraph{Plane Vibration}
The dataset is consisted of time 0 to 73627, with five attributes recorded per timestamp. We randomly take out $10\%$ of data to make the time series irregularly sampled. Following the implementation of \citep{xia2021heavy}, we use the first $50\%$ of data as our train set, the next $25\%$ as a validation set, and the rest as a test set. We divide each set into partitions of consecutive 64 timestamps of the irregularly-sampled time series, and our goal is to forecast $8$ consecutive timestamps starting from the last timestamp of the segment.





\subsection{Network Structure} \label{appen: network_detail}

\paragraph{Time Embedding} 
In order to obtain $\varphi(t)$ and $\phi(t)$, 
we first adopt a positional embedding of transformer \citep{vaswani2017attention}.
A time encoding function converts an one-dimensional time into a multi-dimensional input by passing the time $t$ through trigonometric functions of varying frequencies: 
\[
t\mapsto\left(\sin\left(\omega_{i}t\right),\cos\left(\omega_{i}t\right)\right),\ \omega_{i}=10^{-\frac{4i}{m}} ,
\]
where $i$ runs over a range of integers $\left\{ 0,\dots,m-1\right\} $ with encoding dimension $m$.
Then, we pass MLP layers with the number of hidden dimension $c$. The contextual visualization of sharing net $\varphi(t)$ can be found in Figure \ref{fig:multihead}.



\paragraph{Fourier kernel}
Time modulation multiples a single value $\varphi_\ell\left(t,\xi\right)\in\bR$ to a $d_v\times d_v$ Fourier kernel $R_\ell\left(\xi\right)$, that is,  
\[
R_\ell\left(t,\xi\right)=\varphi_\ell\left(t,\xi\right)\cdot R_\ell\left(\xi\right).
\]
However, often there are multiple different aspects channel elements attend to, and scalar multiplication may not be a good option for it.
To attempt to lift this restriction, we employ parallel heads. Specifically, 
we divide the Fourier kernel $R_\ell$ into $h$ kernel blocks $R_{\ell}^{(i)}\left(\xi\right)\in\bR^{d_k\times d_v}$ for $i=1\dots,h$, where $h$ is the number of heads and $d_k=d_v/h$. Afterward, we concatenate the heads and combine them with a final kernel matrix. (See Figure \ref{fig:multihead}.) This operation can be expressed as:
\begin{align*}
R_\ell\left(t,\xi\right) & = \text{Concat}\left(R_\ell^{(1)},\dots,R_\ell^{(h)}\right), \\
R_\ell^{(i)}\left(t,\xi\right)&=\varphi_\ell^{(i)}\left(t,\xi\right)\cdot R_\ell^{(i)}\left(\xi\right),\ \ i=1,\dots,h.
\end{align*}
The multi-head kernel allows the model to share different representations at different channels.

\begin{figure}
    \centering
    \includegraphics[page=2,width=0.90\textwidth]{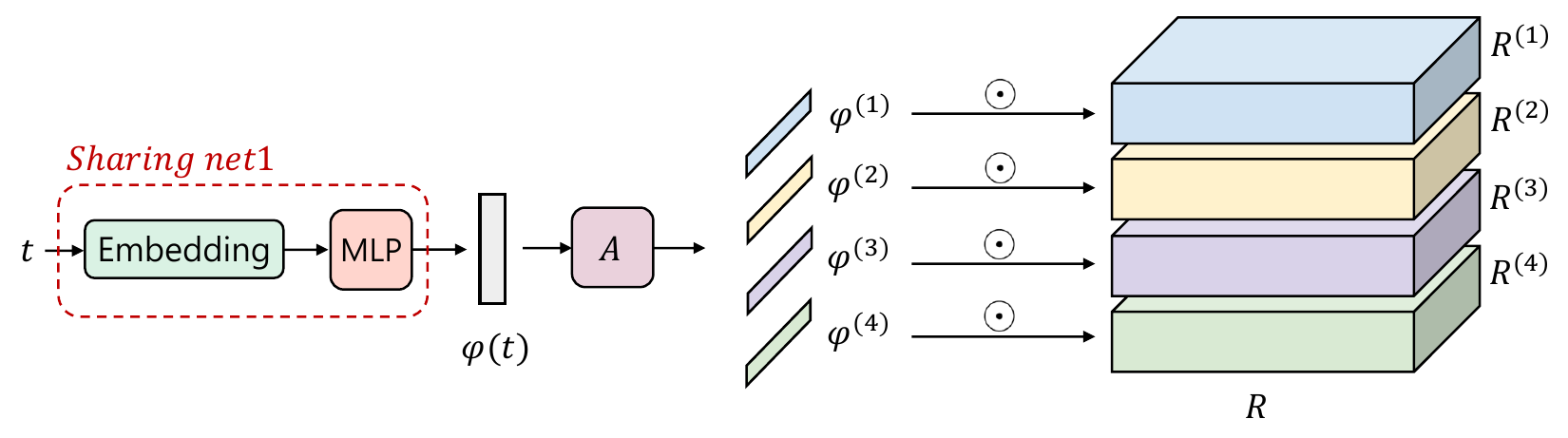}
    \caption{The architecture of multi-head CTFNO.}
    \label{fig:multihead}
\end{figure}

\paragraph{Model architecture for 2D data} 
While the Fourier kernel $R$ plays a key role in FNO models for learning the solution operator of PDEs, it requires quite many parameters ($d_v^2$ for each frequency).
This parameter redundancy worsens for two-dimensional data; images in \cref{appen:classification}.
To overcome parameter inefficiency, we modify the Fourier kernel as follows.
One modification is the use of block-diagonal structure of $R$, as proposed in \citep{guibas2021adaptive}.
This reduces the number of parameters, however, the diagonal structure does not fully mix the information between channels, which may lessen the expressive power of the model.
To compensate for this, we additionally introduce a local $1\times1$ convolution layer $V$ in Fourier space. The overall architecture can be written as
\begin{equation}
    \mathcal{L}_{\ell}\left(v\right)\left(x\right)  =\sigma\left(W_{\ell}v\left(x\right)+\mathcal{F}^{-1}\left(R_{\ell}\left(\xi\right)\cdot\mathcal{F}\left(v\right)\left(\xi\right)+V_\ell\mathcal{F}\left(v\right)\left(\xi\right)\right)\left(x\right)\right). \\
\end{equation}
We empirically show that this improves the training ability of FNO with significantly less parameters.

\paragraph{Lifting Layers for Real Application}
Fourier transform may be hard to embody discontinuous functions because Fourier filters are global sinusoidal functions.
As a consequence, FNO has difficulty in learning sharp features \citep{lu2022comprehensive}.
In many applications, data is discontinuous, such as images.
Hence, the performance of FNO models is plagued by the discontinuous nature of such input data.
Contrary to regular PDE problems where the pointwise lifting is adequate, we hereby propose an alternative lifting layer to provide a smoother form of the data.
A low pass filter is the standard way of smoothing out the data; it tends to retain the low frequency modes while truncating out the high frequency information.
To cutoff high spatial frequency, we additionally introduce a low pass filter after the original lifting layer. 
For parameter efficiency, the low pass filter is implemented by an $1\times1$ convolution layer in Fourier space, as we introduced in the preceding paragraph.  
In this way, input data is smoothed by decreasing the disparity between neighboring values without loss of information.

\paragraph{Ablation Models}
\begin{figure}[t]
    \centering
    \includegraphics[page=1,width=0.90\textwidth]{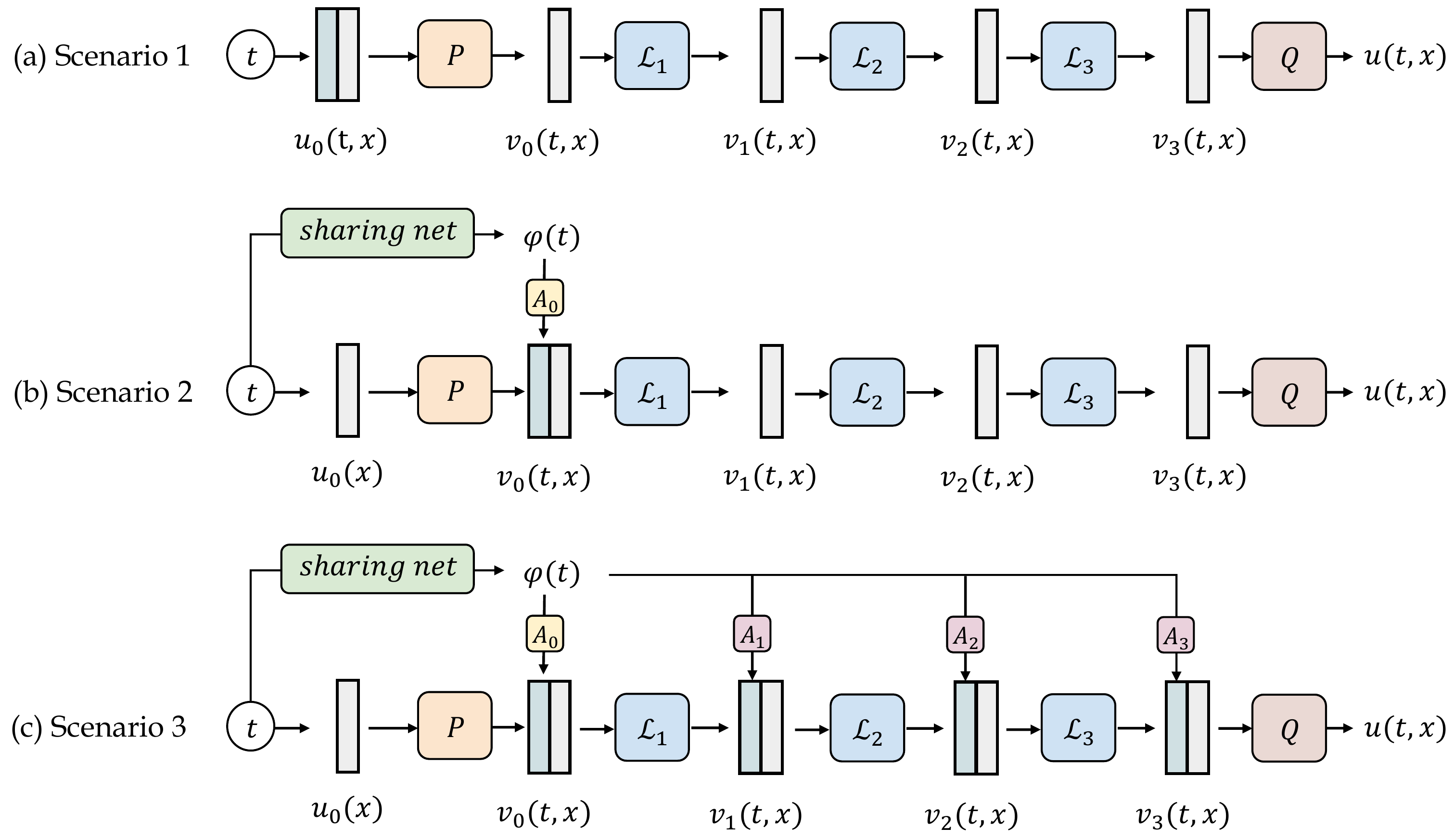}
    \caption{The visualization of architecture of three baselines: (a) Concatenates the input time $t$ to an input function. (b) Concatenates the encoded time $\varphi(t)$ to the lifted input function. (c) Concatenates the encoded time to the lifted input function and intermediate features. Here, all \fbox{A} denotes a learnable affine transform.}
    \label{fig:TFNO_ablation}
    \vspace{-10pt}
\end{figure}

In Section \ref{sec:tfno}, we construct three baselines injecting time $t$ into the network for comparative analysis of CTFNO. This section includes implementation details.
Since CTFNO puts $t$ into weights of the Fourier layers, the other way to inject $t$ is to join it in the input function $u_0(x)$ or intermediate features $v_0(x), v_1(x), \ldots, v_L(x)$. To cover variants of such cases, we constructed three other baselines adding $t$ into the input or feature space as discussed in Section \ref{sec:tfno}.
The architecture of these three scenarios is pictorized in Figure \ref{fig:TFNO_ablation}.
We choose the hyperparameters and other experimental details that delivered the best performance for the baseline models. Note that each model has different optimal optimizer settings, so we carefully choose the learning rate and its decay rate that show the best performance. 

\subsubsection{Synthetic}

\paragraph{Low dimensional data}
Low dimensional data includes Spiral, Stiff, Sawtooth, and Square. 
All the experiments on these datasets are trained for 1000 epochs with a batch size of 120.
We do not normalize data since all data are properly bounded.
Implementation of NODE, ANODE, Latent ODE, Neural Flow, and Neural Laplace follow \citep{holt2022neural}, other than the data normalization.
For CTFNO, a single data point is linearly transformed into a 16-dimensional vector.
Then, we regard the temporal dimension of 100 as a channel size and convert it to 16 by a lifting layer.
We use the learning rate of 1e-3 with no decay.
We do not use the stabilization constant and use Gaussian error Linear Units (GeLU) \citep{hendrycks2016bridging} activation in all FNO-based experiments. 
In the time-embedding layer, we employ Sigmoid Linear Units (SiLU) \citep{elfwing2018sigmoid} activation function.
Note that all experiments for CTFNO use the same activation function.

\paragraph{High-dimensional data}
In Reaction ODE and PDE experiments, which are high dimensional datasets, we set comparison models to have a large number of parameters, at least 1M.
In NODE and ANODE experiments, we set an ODE function $f(t, u(t))$ to a five layer Multi-Layer Perceptron (MLP), of 1024 units.
For ANODE, we set the augmented dimension of zeros to 100.
We employ the learning rate of 1e-5, and the number of epochs of 2000.
Other hyperparameters follow the NODE implementation of \citet{holt2022neural}.
Moreover, in the implementation of Neural Laplace, we set the latent dimension and the number of hidden units to 64.
Other hyperparameters are set the same as low-dimensional data experiments.
In the implementation of Neural Flow, we use a 32-layered coupling flow.
Each coupling flow has two hidden layers with 128 hidden dimensions.
Timesteps are embedded into 32 sinusoidal features.
We set the learning rate of 1e-3 with the decay rate of 0.8 for every 200 epochs.
Other hyperparameters for Neural Flow follow the synthetic experiments in \citep{bilovs2021neural}.
In DON and PDN experiments, we use the number of channels of 256 with six residual layers for each branch and trunk net.
We use SiLU activation, the training epochs of 20000, the batch size of 50, and the learning rate of 1e-3.
The hyperparameters of FNO-based models, including CTFNO, are summarized in Table \ref{tab:synthetic_hparams}. 
We train FNO-based models with a batch size of 20, the training epochs of 2000, and the learning rate of 1e-3 with a decay rate of 0.8 for every 100 epochs.

\begin{table}[t] 
  \caption{The number of parameters of each model in synthetic experiments. Low indicates the experiments on Spiral, Stiff, Sawtooth, and Square. High denotes the high dimensional experiments, including Heat and Burgers' equations. LODE, NF, NL, DON, and PDN stand for Latent ODE, Neural Flow, Neural Laplace, DeepONet, and POD-DeepONet, respectively.}
  \vspace{5pt}
  \centering
   \setlength\tabcolsep{5.0pt}
  \begin{tabular}{ccccccccc}
    \toprule
    Data & NODE & ANODE & LODE & NF& NL& DON & PDN & CTFNO \\
    \midrule
    Low     & 17025  &  17282 & 18565  & 18307  & 17194  &   -   &  -    &   17858 \\
    High    & 5.24M  &  5.45M &  -     & 16.57M   & 4.51M   &   1.58M   & 1.78M  &   2.38M \\
    \bottomrule
  \end{tabular}
  \label{tab:synthetic_nparams}
\end{table}

\begin{table}[t] 
  \caption{Hyperparameter settings of CTFNO, FNO-2D, and FNO-RNN on synthetic data. Low indicates the experiments on Spiral, Stiff, Sawtooth, and Square data. Layers, Modes, and Channels are the number of Fourier layers, Fourier modes, and channels, respectively. The first value in Time channels is the dimension used for time embedding and the second value is the dimension of sinusoidal embedding of time. Params denotes the size of model parameters.}
  \vspace{5pt}
  \centering
   \setlength\tabcolsep{5.0pt}
  \begin{tabular}{ccccccc}
    \toprule
    Model   & Data      & Layers& Modes & Channels  & Time channels & Params \\
    \midrule
     CTFNO  & Low       &    3  &    4  &    16     & (32, 16)    &  17858   \\
            & Reaction  &    2  &   32  &   64      & (512, 128)  &  1.82M \\
            & PDEs       &    2  &   64  &   64      & (512, 128)  &  2.38M \\
     FNO-2D  & Reaction  &   3   & (32, 16) & 32     &  -        &   7.00M\\
            & PDEs       &   3   & (64, 16) & 32     &  -        &   12.00M\\
     FNO-RNN & PDEs       &   4   &   64  &   64      & -         &   2.02M\\
     
    \bottomrule
  \end{tabular}
  \label{tab:synthetic_hparams}
\end{table}

\subsubsection{Real time-series}

Our implementation builds on open-source codes \footnote{\url{https://github.com/rtqichen/torchdiffeq}}\footnote{\url{https://github.com/YuliaRubanova/latent_ode}}\footnote{\url{https://github.com/mbilos/stribor}}\footnote{\url{https://github.com/hedixia/HeavyBallNODE}} (MIT License).
For all real time-series datasets, we refer to the results of RNN-VAE, ODE-RNN and Latent ODEs reported in \citep{rubanova2019latent}.
For CTFNO, we commonly use the number of heads of one and two padding dimensions.
Note that all layers are equipped with GeLU activation, except for only time-embedding layers with SiLU activation.
Other model hyperparameters for CTFNO are shown in Table \ref{tab:real_hparams}.
Moreover, we use the batch size of 50, and the Adamax optimizer \citep{kingma2014adam} with the learning rate 1e-2.


\paragraph{MuJoCo}
Overall experimental settings follow \citep{rubanova2019latent}, unless stated.
We use the latent dimension of 20, the hidden state dimension of 100 for the RNN encoder, and the training epochs of 300.
Since \citet{bilovs2021neural} did not contain the results of interpolation and prediction tasks with different subsampling ratios, we carefully trained Neural Flow, following the official implementation\footnote{\url{https://github.com/mbilos/neural-flows-experiments}}.
Following \citep{rubanova2019latent}, we employ the importance weighted likelihood loss \citep{burda2015importance}.

\paragraph{PhysioNet}
Overall experimental settings follow \citep{rubanova2019latent}, unless stated.
We use the latent dimension of 20 and the hidden state dimension of 40 for the RNN encoder.
The training epochs are 200 and 50 for the interpolation task and the prediction task, respectively.
As in \citep{rubanova2019latent}, we employ the importance weighted likelihood loss \citep{burda2015importance} without leveraging classification loss.

\paragraph{Activity}
Overall experimental settings follow \citep{rubanova2019latent}, unless stated.
We use the latent dimension of 20, the hidden state dimension of 100 for the RNN encoder, and the training epochs of 200.
For classification, the cross entropy loss was used, as in \citep{rubanova2019latent}.

\begin{table}[t] 
  \caption{Hyperparameter settings of CTFNO on real-world time-series data. Layers, Modes, and Channels are the number of Fourier layers, Fourier modes, and channels, respectively. The first value in Time channels is the dimension used for time embedding and the second value is the dimension of sinusoidal embedding of time. $M$ stands for the stabilization constant. Params denotes the number of model parameters.}
  \vspace{5pt}
  \label{tab:real_hparams}
  \centering
   \setlength\tabcolsep{5.0pt}
  \begin{tabular}{ccccccc}
    \toprule
    \multicolumn{1}{c}{Data} & \multicolumn{1}{c}{Layers} & \multicolumn{1}{c}{Modes} &\multicolumn{1}{c}{Channels} & \multicolumn{1}{c}{Time channels} & \multicolumn{1}{c}{$M$} & \multicolumn{1}{c}{Params} \\
    \midrule
    MuJoCo & \multicolumn{1}{c}{$3$} & \multicolumn{1}{c}{$10$} & \multicolumn{1}{c}{$32$} & \multicolumn{1}{c}{$(64,32)$} & \multicolumn{1}{c}{$10$} & \multicolumn{1}{c}{$150$K}\\
    PhysioNet & \multicolumn{1}{c}{$1$} & \multicolumn{1}{c}{$10$} & \multicolumn{1}{c}{$32$} & \multicolumn{1}{c}{$(64,8)$} & \multicolumn{1}{c}{-} & \multicolumn{1}{c}{$68$K} \\
    Activity & \multicolumn{1}{c}{$2$} & \multicolumn{1}{c}{$10$} & \multicolumn{1}{c}{$32$} & \multicolumn{1}{c}{$(64,32)$} & \multicolumn{1}{c}{10} & \multicolumn{1}{c}{$124$K} \\
    Plane Vibration & \multicolumn{1}{c}{$2$} & \multicolumn{1}{c}{$10$} & \multicolumn{1}{c}{$8$} & \multicolumn{1}{c}{$(16,8)$} & \multicolumn{1}{c}{$1.2$} & \multicolumn{1}{c}{$6$K}   \\
    \bottomrule
  \end{tabular}
\end{table}

\paragraph{Plane Vibration}
In this experiment, we follow the settings of \citet{xia2021heavy}.
For CTFNO, five-dimensional attributes are concatenated with input time and then embedded into 20-dimensional space by a linear transformation.
We regard temporal dimension 64 as a channel size and convert it to eight by a lifting layer.
We train the model for 500 epochs and use the learning rate of 1e-2 with decay rate 0.5 per 100 epochs.
We trained and evaluated with MSE loss.

\section{Further Results}

\subsection{Stability on Image Classification} \label{appen:classification}
We perform experiments on image classification, which is one of the benchmark applications of Neural ODEs. Although this is not a time-dependent problem, we include classification to validate the effect of the proposed stabilization scheme. 
Images could be considered as functions $a:\left[0,1\right]^2\rightarrow\bR^3$ of light defined on a continuous region of pixel locations, instead of $32\times32$ pixel vectors of RGB values.
Classification requires us to learn a dynamics that simultaneously drives each input $a\left(\mathbf{x}\right)$ to the corresponding final feature $u\left(\mathbf{x}\right)$, allocated depending on its label. 
Therefore, posing the problem as learning an operator that maps the image $a\left(\mathbf{x}\right)$ to $u\left(\mathbf{x}\right)$, we can apply FNO to image classification.

We compare the results with those of CNN and ANODE as a baselines using two benchmark 
image classification problems: 
MNIST \cite{lecun1998gradient} and CIFAR10 \cite{krizhevsky2009learning}.
All models are fit to share the similar number of parameters.

\begin{table}[h]
  \caption{Classification accuracy (\%) and robustness of different models on MNIST.}
  \label{tab:mnist}
  \vspace{5pt}
  \centering
  \setlength\tabcolsep{10.5pt}
  \scalebox{0.92}{
  \begin{tabular}{lccccccc}
    \toprule
     & Clean & \multicolumn{2}{c}{Gaussian noise} & \multicolumn{4}{c}{Adversarial attack}                  \\
 \cmidrule(lr){3-4} \cmidrule(lr){5-8}
    \multicolumn{1}{c}{Model} & & $\sigma=50$ & $\sigma=100$ & FGSM-30/255 & FGSM-50/255 & PGD-30/255 & PGD-50/255 \\
    \midrule
    CNN              & \textbf{99.0} & 29.7 & 10.4 & 27.3 & 15.3 & 10.8 & 4.6 \\
    ANODE            & 97.9 & 91.8 & 56.4 & 26.9 & 1.5 & 62.5 & 9.7 \\
    FNO              & 98.6 & 90.2 & 45.6 & 36.9 & 25.8 & 8.1 & 4.1 \\ 
    FNO (WD) & 98.9 & 90.0 & 43.6 & \textbf{57.6} & \textbf{35.6} & 22.1 & 11.7 \\ 
    FNO (Stab) & 98.8 & \textbf{97.8} & \textbf{88.2} & 36.2 & 3.43 & \textbf{62.6} & \textbf{47.1} \\
    \bottomrule
  \end{tabular}}
  \vspace{-5pt}
\end{table}
\begin{table}[h]
  \caption{Classification accuracy (\%) and robustness of different models on CIFAR10.}
  \label{tab:cifar10}
  \vspace{5pt}
  \centering
  \setlength\tabcolsep{12pt}
  \scalebox{0.92}{
  \begin{tabular}{lccccccc}
    \toprule
     & Clean & \multicolumn{2}{c}{Gaussian noise} & \multicolumn{4}{c}{Adversarial attack}                  \\
 \cmidrule(lr){3-4} \cmidrule(lr){5-8}
    \multicolumn{1}{c}{Model} & & $\sigma=15$ & $\sigma=20$ & FGSM-2/255 & FGSM-5/255 & PGD-2/255 & PGD-5/255    \\
    \midrule
    CNN     & \textbf{75.5} & 60.0 & 51.1 & 23.0 & 4.7 & 16.7 &    0.43  \\
    ANODE            & 59.4 & 40.9 & 34.9 & 0.1    & 0.0   & 0.0   & 0.0  \\
    FNO         & 64.0 & 57.9 & 52.2 & 3.4 & 1.4 & 1.5 & 0.0  \\ 
    FNO (WD) & 66.9 & 54.9 & 46.2 & 5.8 & 0.6 & 2.0 & 0.0 \\ 
    FNO (Stab) & 65.5 & \textbf{64.3} & \textbf{63.3} & \textbf{38.6} & \textbf{11.1} & \textbf{38.1} & \textbf{9.1} \\
    \bottomrule
  \end{tabular}}
  \vspace{-5pt}
\end{table}

\paragraph{Robustness on Adversarial Attack}
By ensuring the stability of the learned PDE in Section \ref{sec:Stab}, we may expect that the resulting solution for input with a small perturbation converges to the same label as the unperturbed one. 
To investigate this stabilization effect on FNO, 
we test the trained model in defending against Gaussian noise, fast gradient signed method (FGSM) \citep{goodfellow2014explaining}, and projected gradient descent (PGD) \citep{madry2017towards} attack. 
Weight decay (WD) is a standard practical technique to regularize the $L^2$ norm of weights of models.
Because WD has a similar intention to our stabilization scheme, these two are also compared.
For a fair comparison, CNN and ANODE are trained with WD.
Tables \ref{tab:mnist} and \ref{tab:cifar10} show that for both datasets, our stabilization significantly enhances the robustness of the model. 
Our model achieves better results in both accuracy and robustness than ANODE. Although CNN produces around 10\% higher accuracy in CIFAR10, our model outperforms CNN for perturbed images.
Moreover, the results indicate that FNO trained with WD is prone to mislead by adversarial attacks. 
This demonstrates that it is much more accurate to apply a model-suitable stabilization based on the PDE theory, rather than naively regularizing the $L^2$ norm of the weights.
\paragraph{Effect of $M$} \label{appen:stab}
As we discussed in Section \ref{sec:Stab}, the upper bound $M$ on the  $\left\Vert R\left(t,\cdot\right)\right\Vert _{2}$ and $\left\Vert W\left(t,\cdot\right)\right\Vert _{2}$ emanates from the necessity of stabilizing the proposed model.
$M$ is a user-defined hyperparameter. Here, we analyze the performance of FNO to illustrate how increasing stabilization parameter $M$ leads to a more robust model. Experiments are carried out on CIFAR10 and we perturb images by PGD-5/255 attack.
The left panel of Figure \ref{fig:acc} shows the accuracy and robustness of the learned models for several values of $M$.
We can see that imposing stability does come with trade-offs between test accuracy and robustness on the attack.
The accuracy is higher at low $M$ and it becomes decreasing as $M$ increases, which makes the model more robust on various attacks. Moreover, the models with weak stabilization show brittle training procedures.
Also, the model degrades the overall performance as stabilization effects become dominant.
Results certify that small corruptions or extra noise in the input are not likely to change the output of the network with proper stabilization.

Moreover, the stability of the model is related to how well the model generalizes on the data on which it has not trained, which in turn is related to overfitting.
This can be confirmed in Figure \ref{fig:acc} (right).
It presents the evolution of the loss function of models with and without stabilization.
As one can see, after some iterations, the test accuracy of the model without stabilization has started to decrease.
This means that the longer we train the model, the more specialized the weights will become to the training data. It is evidence of overfitting.
On the other hand, the model with stabilization keeps increasing its test accuracy.
This confirms that the stabilization prevents overfitting and helps the model to work better on unseen data.

\begin{figure}[h]
	\centering{}
         \includegraphics[width=0.99\textwidth]{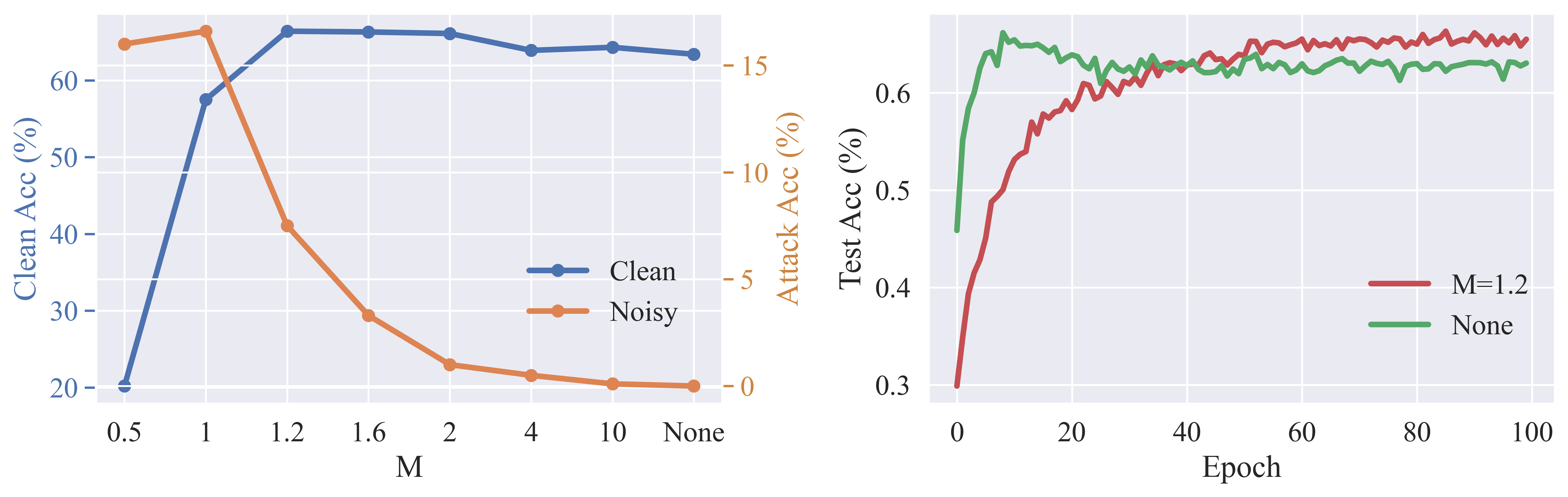}
\caption{The Trade-off between the accuracy and robustness on PGD-5/255 attack (left) and observation for overfitting (right). Tests are conducted on CIFAR10 (\cref{sec:exp_stab}).}
\label{fig:acc}
\end{figure}

The relation between the stabilization and overfitting was also discussed in Section \ref{sec:exp_stab} on the Plane Vibration dataset. By increasing $M$, the training MSE loss gradually decreases, while the test MSE eventually goes up (see Figure \ref{fig:pv}). Figure \ref{fig:pv_stab} plots the distribution of the $L^2$ norm of weights of trained CTFNO with $M=1.2$ and $10$ on the plane vibration data.
The weights of the model trained with $M=10$ have larger $L^2$ norms. 
As we analyzed in Theorem \ref{prop:stab}, the weights that have a large spectrum make the network unstable.
More precisely, the spectrum of the weights grows in size to handle the specifics of the training data. 
As the weights become specialized to the training data, overfitting occurs.
On the other hand, the trained model with $M=1.2$ learned weights with a spectrum smaller than 1. It forces the network to have small changes in output for small changes in the inputs, which gives more ability to generalize better.

\begin{figure}[h]
    \centering
    \includegraphics[width=0.99\textwidth]{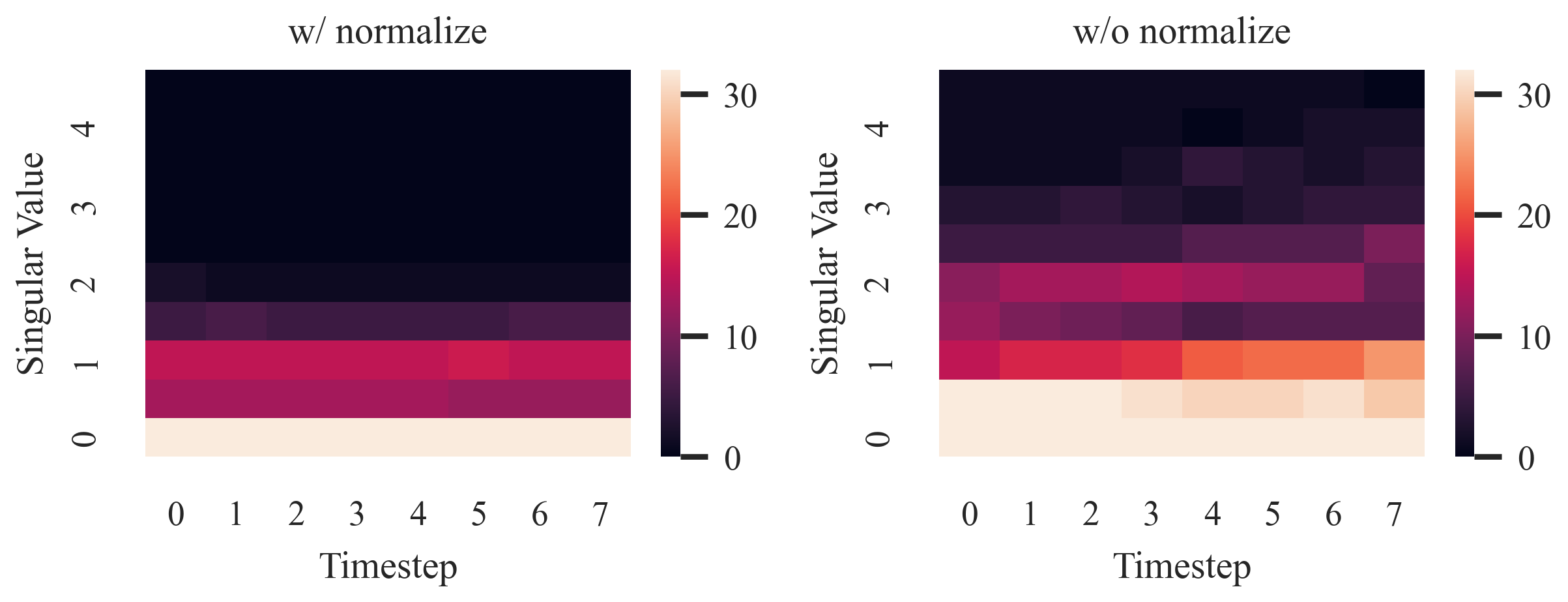}
    \caption{Plot of the singular values of weights of CTFNO with (left) and without (right) stabilization on the plane vibration dataset (\cref{sec:exp_stab}).}
    \label{fig:pv_stab}
\end{figure}

\paragraph{Generalization to resolution}
\begin{wraptable}{r}{0.40\textwidth}
    \centering
    \vspace{-15pt}
     \setlength\tabcolsep{3.0pt}
     \caption{Resolution invariance of FNO.} \label{tab:resol_invariance}
     \vspace{5pt}
    \begin{tabular}{ccccc}
    \toprule
    Dataset & Original & \multicolumn{3}{c}{Resolution}\\
    \cmidrule(lr){3-5}
         &  & $\times 2$ & $\times 4$ & $\times 8$ \\
    \midrule
        MNIST   &98.8 & 98.6 & 98.6 & 98.6\\
        CIFAR10 & 65.5 & 64.8 & 64.6 & 64.7 \\
    \bottomrule
    \end{tabular}
    \vspace{-5pt}
\end{wraptable} 
Classical neural networks that map between finite-dimensional Euclidean spaces are grid-dependent and thus cannot be generalized over image resolution. By viewing images as functions, FNO is not tied to a specific discretization and can be applied to arbitrary discretization of $a\left(\mathbf{x}\right)$.
To validate the resolution invariance of FNO, we classify high-resolution images using FNO trained on the original low-resolution images.
Table \ref{tab:resol_invariance} validates that FNO is able to learn from coarse images and generalize to higher resolutions.   

\paragraph{Experimental Details}
All of the experiments use gradient clipping of 10 and cross-entropy loss. We use a learning rate of 1e-3 and weight decay of 1e-4 for ANODE and CNN. Other implementation details of ANODE follow \citet{dupont2019augmented}. We use five convolutional layers in CNN experiments. For MNIST, we apply a kernel size of seven for the first two convolutional layers and 32 channels. For CIFAR10, we use a kernel size of seven for the first convolutional layers and 64 channels. Other convolutional layers have a kernel size of three. We train CNN for 30 epochs. Note that the accuracy of both ANODE and CNN tends to decrease after the designated number of epochs.

In FNO experiments, FNO with weight decay (WD) and FNO with normalization (Stab), we train models for 100 epochs with the learning rate of 5e-3.
We used three Fourier layers with six Fourier modes.
For MNIST, we employ the number of channel of 32, and three-dimensional zero padding.
For CIFAR10, we employ the number of channel of 64, and four-dimensional zero padding.
For the experiments with decay, we choose the decay rates 1e-6 and 1e-4 for MNIST and CIFAR10, respectively.
For the experiments with normalization, we employ the spectral norm of 1.2 for the normalization bounds.
Note that the GeLU activation function was used for both CNN and FNO experiments.

To estimate the robustness of the trained model, we consider three commonly-used perturbation schemes, namely random Gaussian perturbations, FGSM, and PGD attacks.
For random Gaussian attack, we use zero mean Gaussian noise with a standard deviation of 50/255, 100/255 for MNIST, and 15/255, 20/255 for CIFAR10.
For adversarial attacks, we use 30/255, 50/255 attacks for MNIST, and 2/255, 5/255 for CIFAR10.
The number of steps of PGD attack is 10.

\begin{table}
    \caption{The number of parameters for each models for image classification.}
    \vspace{5pt}
    \centering
    \begin{tabular}{cccc}
    \toprule
        Data & CNN & ANODE & FNO \\
    \midrule 
        MNIST   & 78K  & 83K  & 66K \\
        CIFAR10 & 154K & 168K & 153K \\
    \bottomrule
    \end{tabular}
\end{table}

\subsection{Speed improvements}
We showed the average wall clock time measured on heat equation in \cref{tab:Time_elapsed_heat}.
Compared to FNO-RNN, which autoregressively predicts solutions, CTFNO takes only 12\% of the training time. Also, CTFNO takes only 1/3 of the training time and 1/6 of inference time of FNO-2D.
In addition, to investigate how much our model reduced computational time in real application, We carefully measured average execution time on the MuJoCo experiment (\cref{sec:real}) and results are reported in Table \ref{tab:mujoco}.
The main reason for this difference is that numerical integration used in LODE is computationally expensive, and Neural Flow and CTFNO are neural operators that directly represent the solution trajectory of ODE and PDE, respectively.
In other words, the neural operator's replacement of the role of ODE solvers has contributed significantly to overcoming the increased complexity of PDEs compared to ODEs and reducing time costs.

\begin{table}[h]
    \centering
    \setlength\tabcolsep{5.0pt}
    \vspace{-5pt}
    \caption{Measured training and inference time on MuJoCo dataset} \label{tab:Time_elapsed}
    \vspace{5pt}
        \begin{tabular}{ccc}
        \toprule
        Model & Training (second/epoch) & Inference (second/epoch) \\
        \midrule
        LODE  & 23.6 & 2.20 \\
        Neural Flow & 4.38 & 0.62 \\
        CTFNO & 7.39 & 0.89 \\
        \bottomrule
        \end{tabular}
\end{table}

\newpage
\subsection{Additional Synthetic Results} \label{appen:heatmaps_ode_heat_burgers}
\paragraph{DE-based models for learning spatial-dependencies}
\begin{wraptable}{r}{0.3\textwidth}
    \centering
    \vspace{-12pt}
      \setlength\tabcolsep{4.2pt}
    \caption{Additional RMSE ($\times10^{-2}$) results on function valued data.} \label{tab:NODE for PDE}
    \vspace{5pt}
     \scalebox{0.95}{
    \begin{tabular}{ccc}
    \toprule
    Model & Reaction & Heat \\ 
    \midrule
    Neural Flow    & 10.300 & 3.038 \\
    Neural Laplace & 2.804 & 2.205\\ 
    \midrule
    FNO-2D      &0.326& 0.0330 \\ 
    CTFNO         & 0.239&  0.0269 \\ 
    \bottomrule
  \end{tabular}}
    \vspace{-5pt}
\end{wraptable} 
In \cref{sec:synthetic}, we investigated the expressivity of DE-based models on various temporal-evolving systems. Here, we additionally show that conventional DE-based models that do not consider the relation between spatial variables result in degraded performance when learning reaction ODE and heat equation. In this case, we only consider operator learning models: Neural Flow and Neural Laplace. 
Compared with FNO-2D and our CTFNO, it can be seen that the PDE surrogates approximate the solution operator of the heat equation much better than the ODE solution operator (Neural Flow) and Laplace-transform based approach (Neural Laplace). In addition, we test FNO-2D on reaction ODE. Contrary to the degraded performance of Neural Flow and Neural Laplace on PDEs, FNO-2D learns the reaction ODE fairly well.

\paragraph{Results on Extrapolation}
We investigate the generalization ability of our method to time points beyond those that were used for training on the two-dimensional spiral data in Section \ref{sec:synthetic}. Predicted and extrapolated trajectories are depicted in Figure \ref{fig:spiral_extrap} together with the ground truth. Each model is trained to predict spiral trajectories at 100-time points (blue lines in Figure \ref{fig:spiral_extrap}). And we run trained models to forecast 500 future time steps (red lines in Figure \ref{fig:spiral_extrap}). We can observe that our CTFNO is better at generalizing for extrapolation compared to baseline models.
Our CTFNO correctly extrapolates the spiral trajectory beyond the training time interval, converging to the equilibrium of the spiral.
From Table \ref{tab:synthetic}, we have seen that CTFNO best approximates the spiral compared to the benchmarks, achieving the lowest RMSE,
The results depicted in Figure \ref{fig:spiral_extrap} confirm that CTFNO accurately captures the dynamics of the spiral during training and predicts well even for non-training times, based on memories analyzed in the past.

\begin{figure}
	\centering{}
         \includegraphics[width=0.7\textwidth]{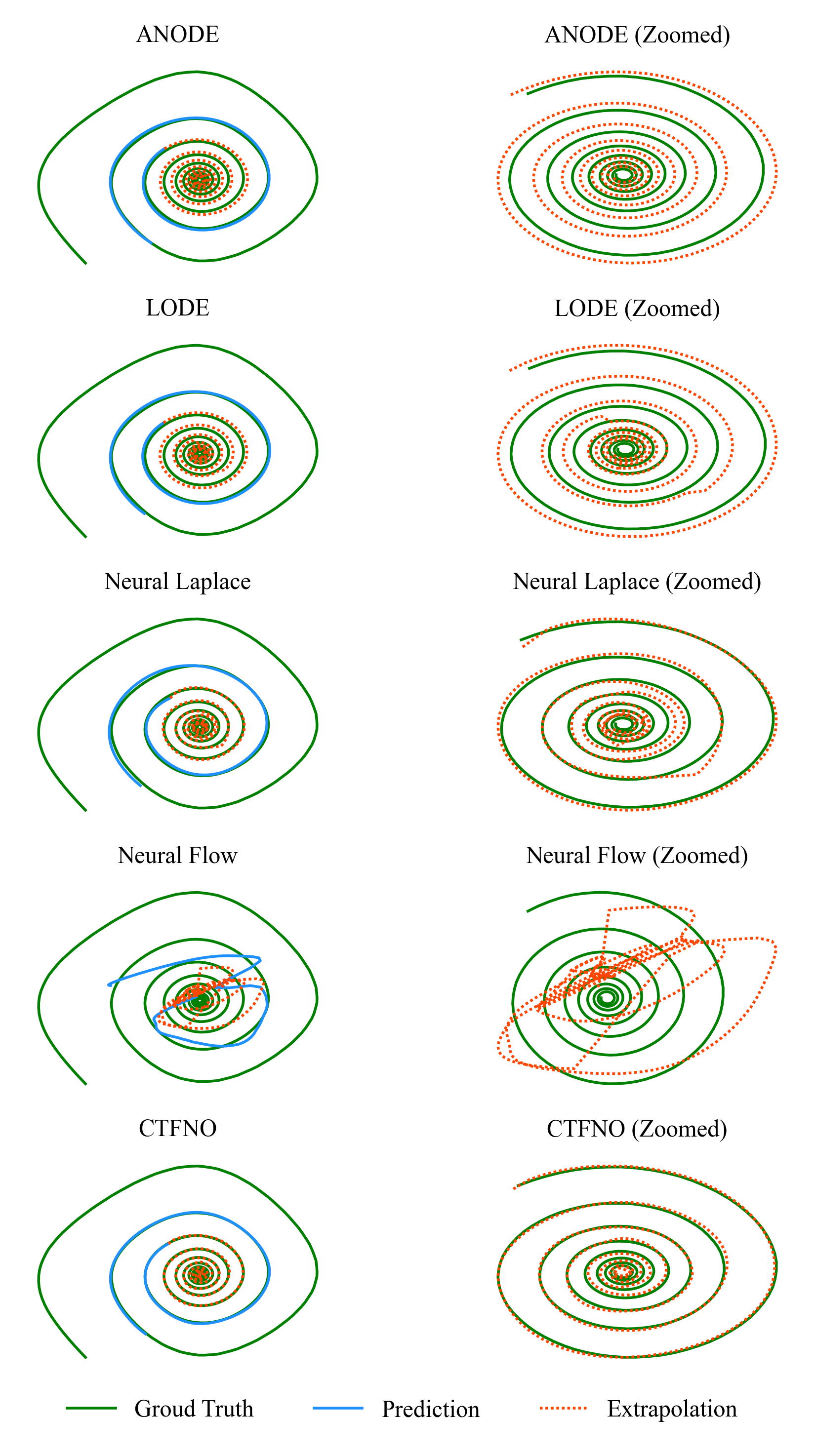}
\caption{(Left) Prediction and extrapolation for beyond the training time points of the two-dimensional spiral (\cref{sec:synthetic}), and (Right) the zoomed extrapolation trajectories with a comparison to ground truth. CTFNO performs best not only for prediction but also for extrapolation.}
\label{fig:spiral_extrap}
\end{figure}

\paragraph{Robust to noisy observations }
We examine the ability of our CTFNO to handle noisy observation.
For experiments, the vanilla and the stabilized CTFNOs with $M=1$ are both trained on sine trajectories and sawtooth data. 
To test the robustness of the models against noise, we deploy the trained models on noisy test data, which is corrupted by additional standard Gaussian noise $\cN\left(0,I\right)$.
\Cref{fig:stab,fig:saw_stab} illustrate the predicted trajectories of both models.  
The sine trajectory considered here has a decreasing amplitude over time. 
Compared to the vanilla model, the perturbation of the test data set does not cause a huge change in the prediction of the stabilized CTFNO.
We can see that the stabilized model recovers the original trajectory much more closely. Similar to the true sine trajectory, the amplitude decreases over time, and the period is constant in \cref{fig:stab}. 
\begin{figure}[h]
\begin{center}
  \includegraphics[width=0.95\textwidth]{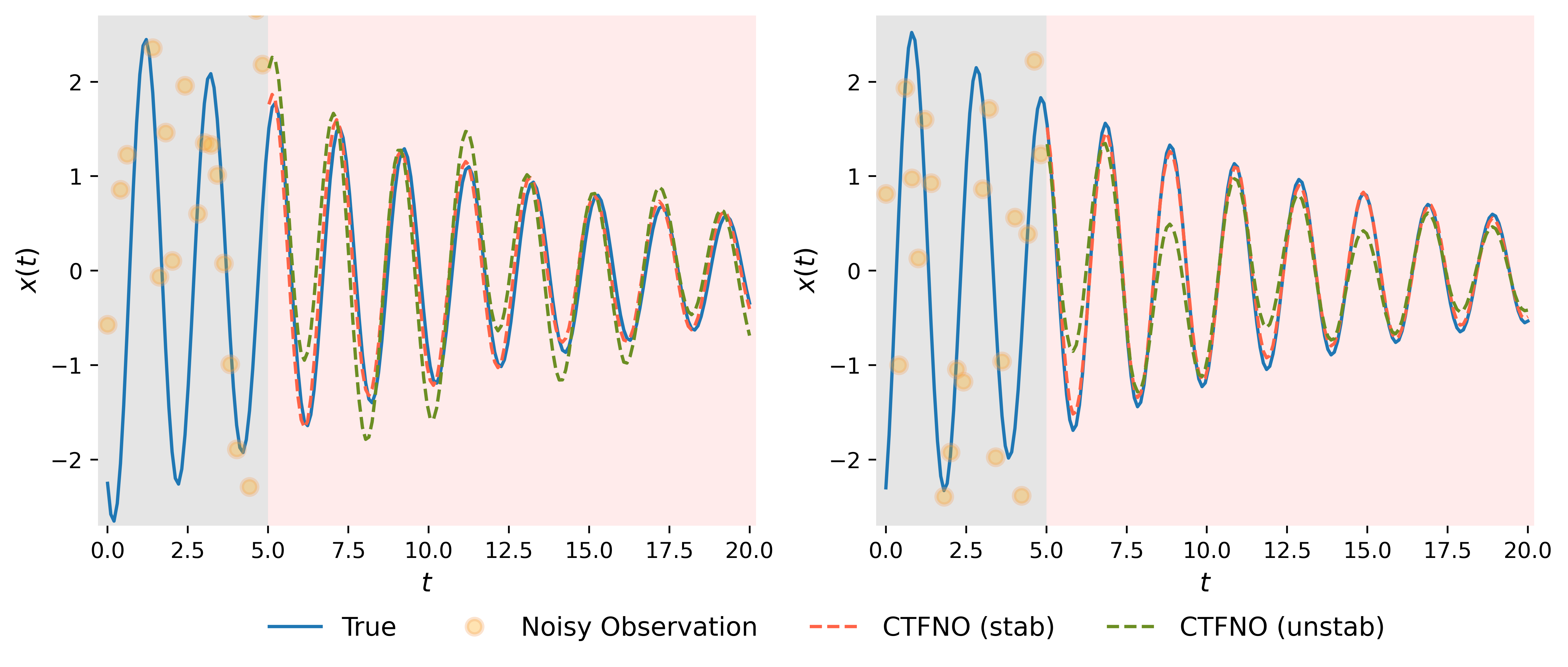}
  \end{center}
  \caption{Predicted sine trajectories for CTFNOs w/ and w/o stabilization over noisy observations.}
  \vspace{-10pt}
  \label{fig:stab}
\end{figure}

\begin{figure}[h]
\begin{center}
  \includegraphics[width=0.95\textwidth]{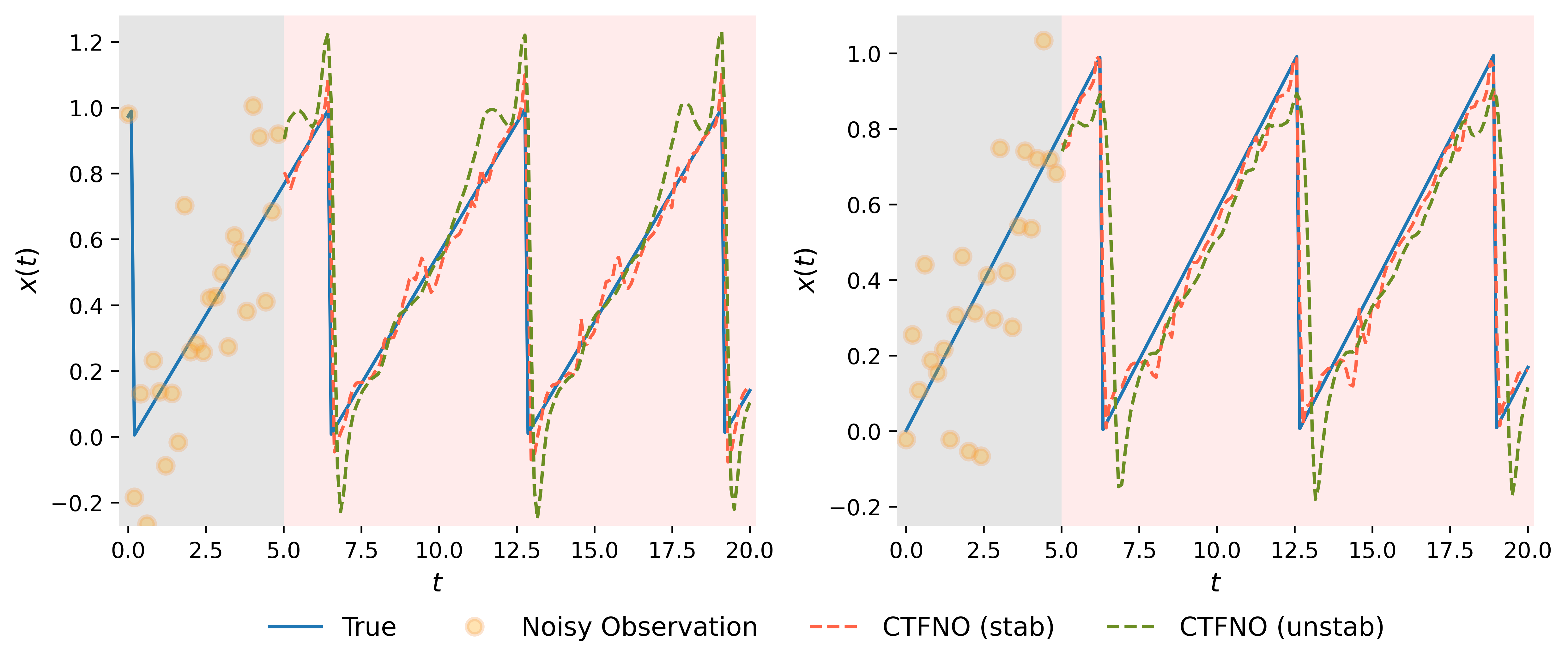}
  \end{center}
  \caption{Predicted sawtooth trajectories for CTFNOs w/ and w/o stabilization over noisy observations.}
  \vspace{-10pt}
  \label{fig:saw_stab}
\end{figure}

\paragraph{Additional Heatmaps}
We also include additional heatmaps for reaction (\cref{sec:synthetic}), heat, Burgers', diffusion-sorption, and compressible Navier-Stokes equations (\cref{sec:pde}). See Figures \ref{fig:ode_total}, \ref{fig:heat_total}, \ref{fig:burgers_total}, \ref{fig:diffsor_total}, and \ref{fig:NS_total}.
\begin{figure}
	\centering{}
         \includegraphics[width=0.93\textwidth]{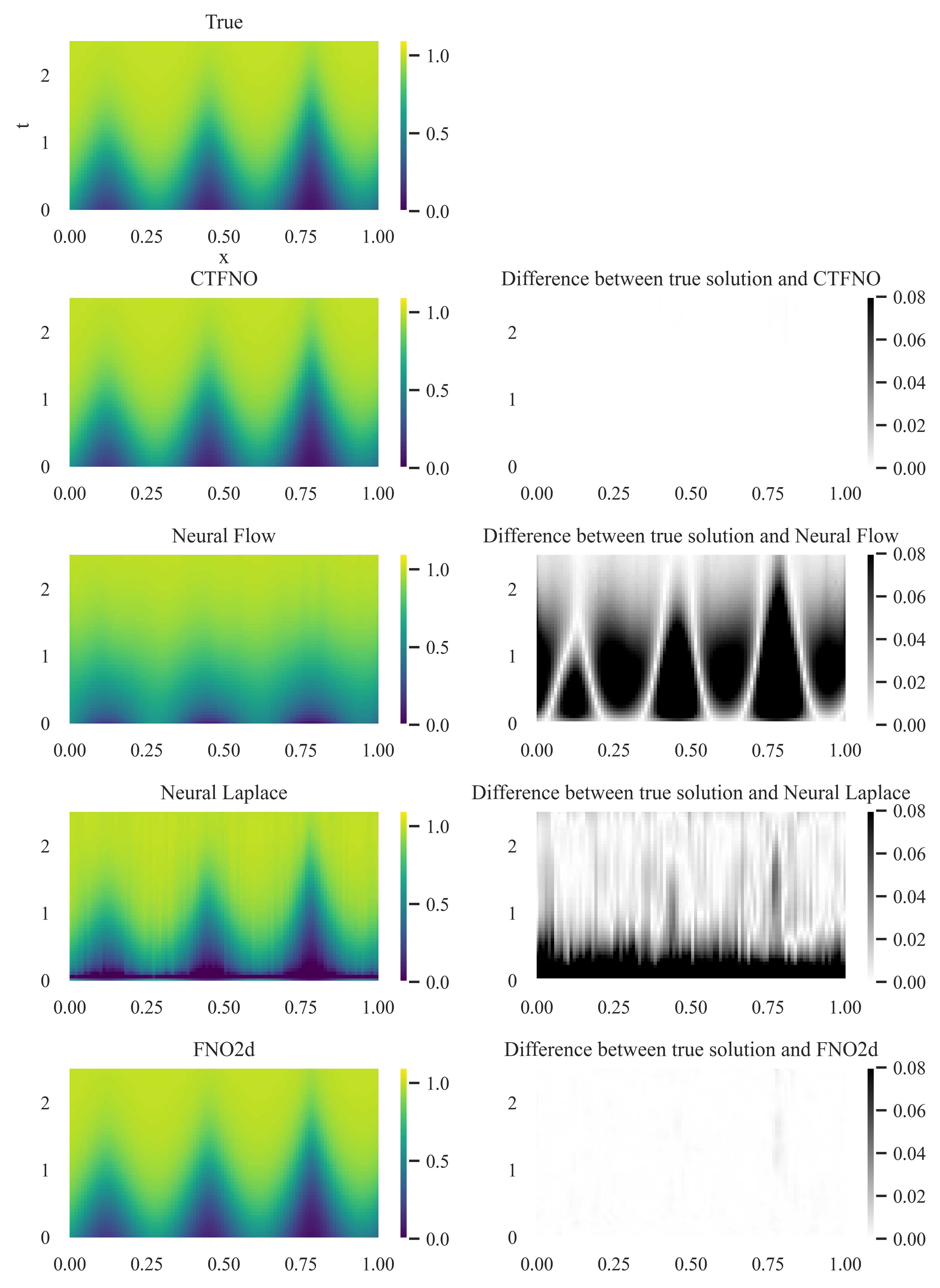}
\caption{Heatmap of exact and predicted solution for reaction ODE (\cref{sec:synthetic}). This shows that our model can exactly represent the dynamics of the reaction ODE.}
\label{fig:ode_total}
\end{figure}

\begin{figure}
	\centering{}
         \includegraphics[width=0.93\textwidth]{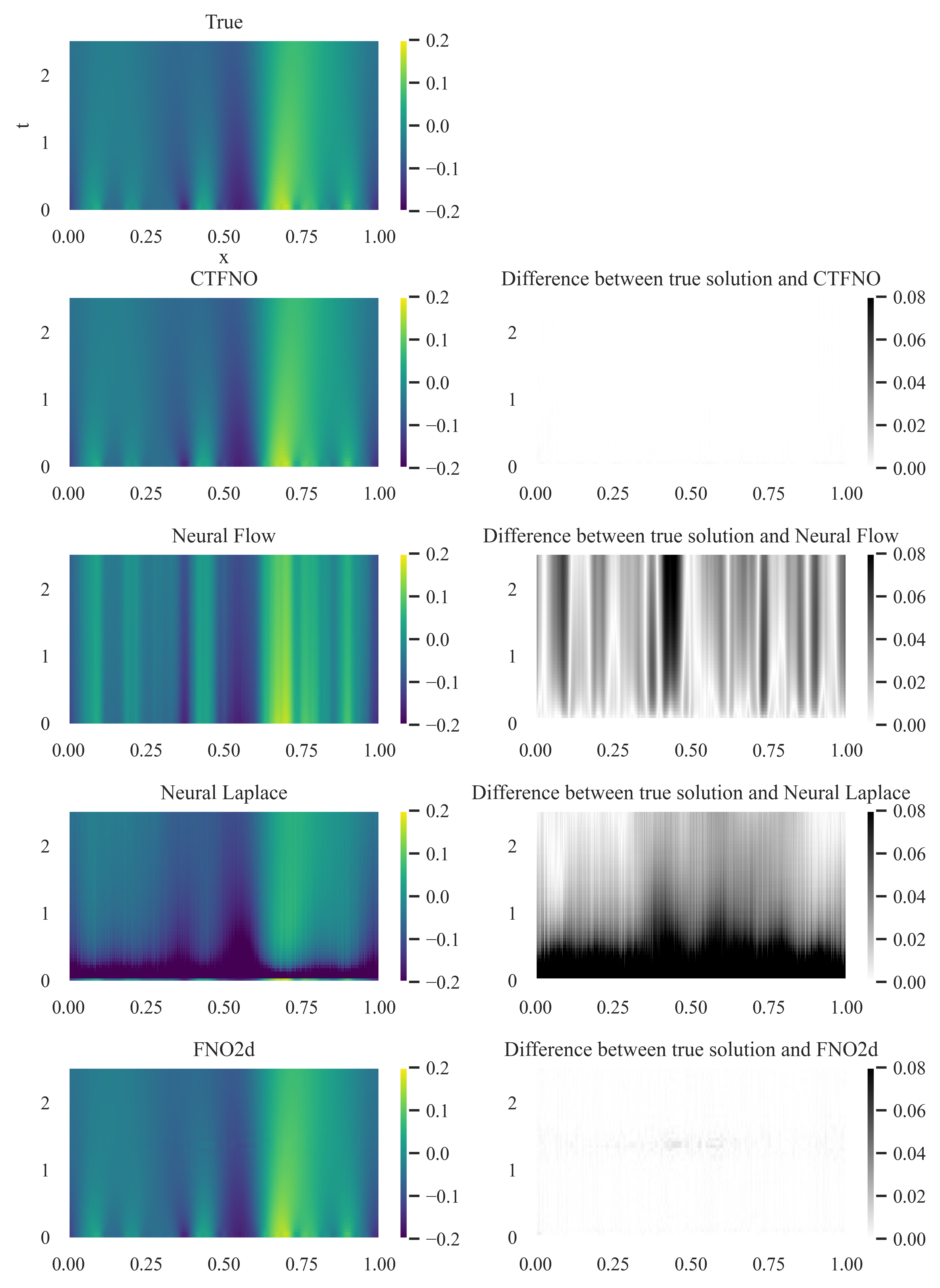}
\caption{Heatmap of exact and predicted solution for heat equation (\cref{sec:pde}).}
\label{fig:heat_total}
\end{figure}

\begin{figure}
	\centering{}
         \includegraphics[width=0.93
         \textwidth]{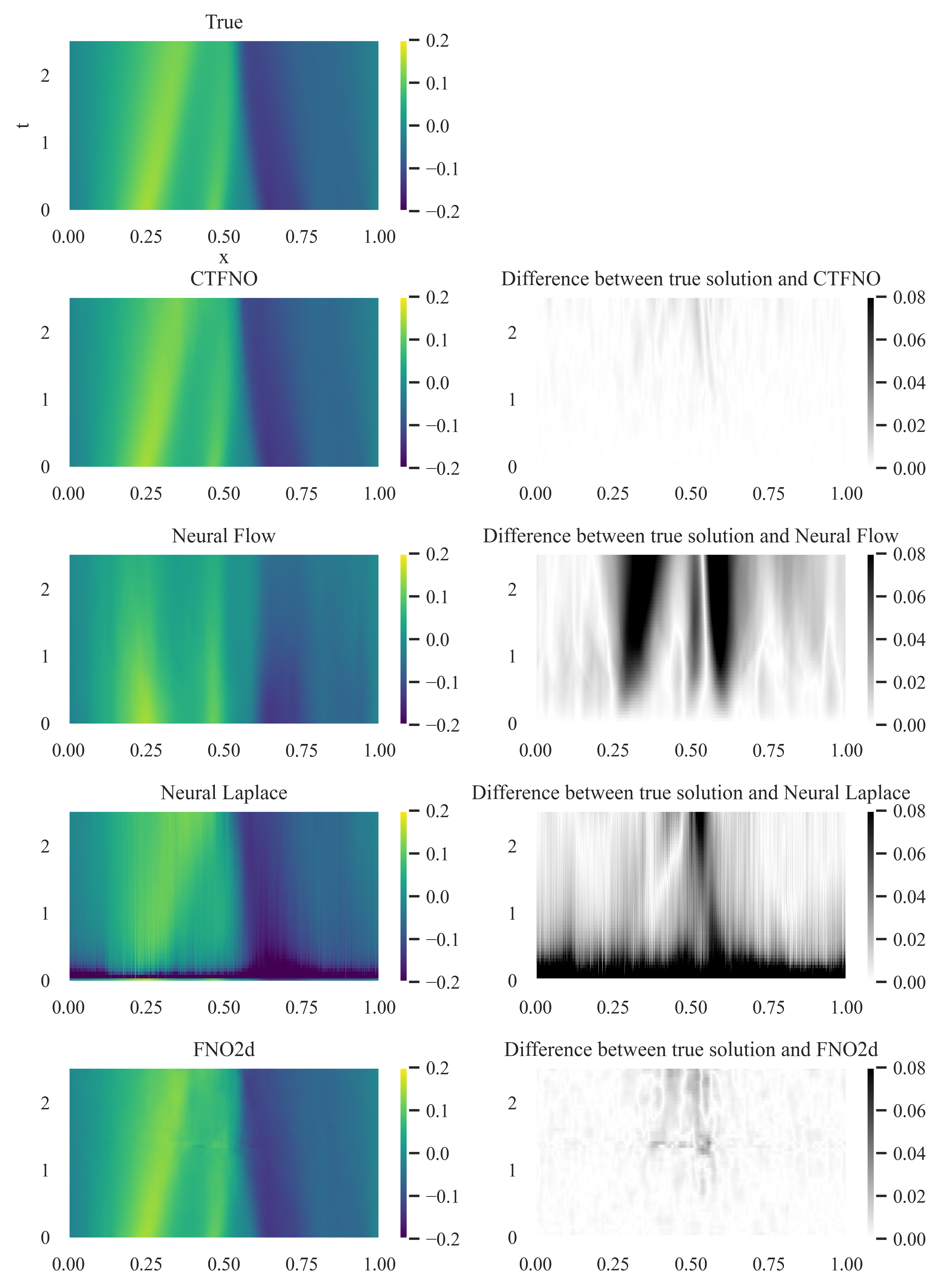}
\caption{Heatmap of exact and predicted solution for Burgers' equation (\cref{sec:pde}).}
\label{fig:burgers_total}
\end{figure}

\begin{figure}
	\centering{}
         \includegraphics[width=0.93
         \textwidth]{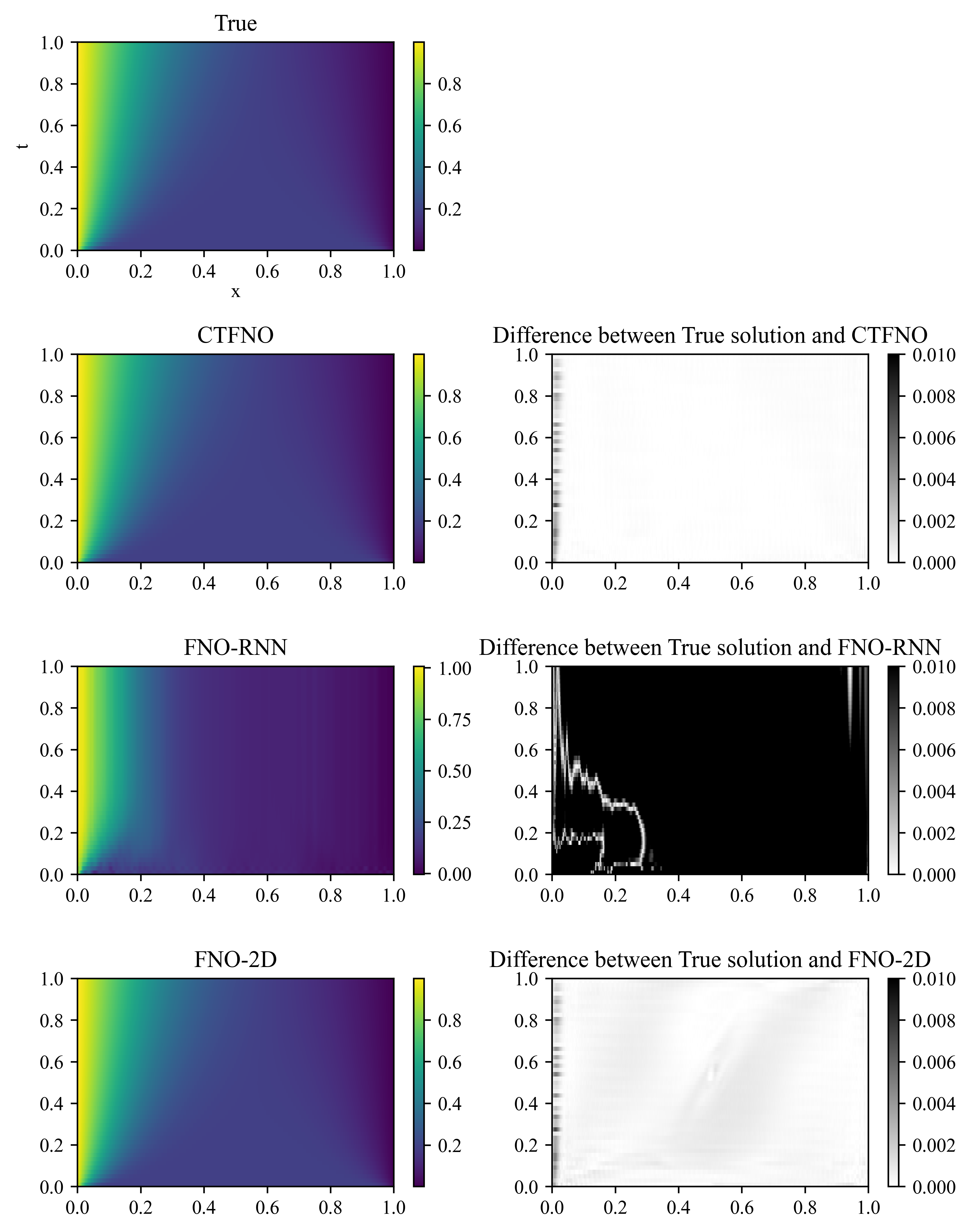}
\caption{Heatmap of exact and predicted solution for diffusion-sorption equation (\cref{sec:pde}).}
\label{fig:diffsor_total}
\end{figure}

\begin{figure}
	\centering{}
         \includegraphics[width=0.93
         \textwidth]{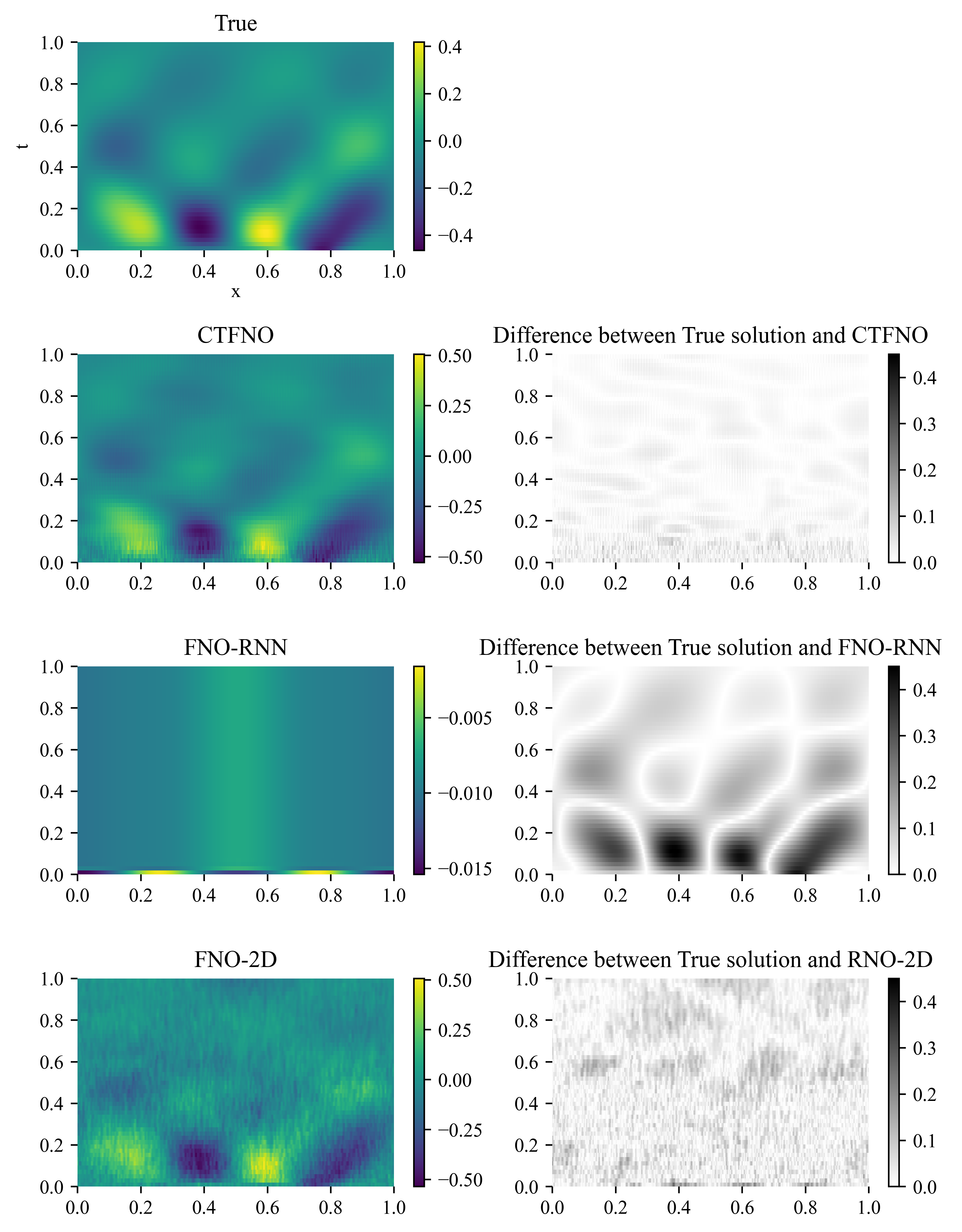}
\caption{Heatmap of exact and predicted solution for compressible Navier-Stokes equations (\cref{sec:pde}).}
\label{fig:NS_total}
\end{figure}

\section{Broader Impact}
We introduced a new framework for modeling temporal dynamics of observed data. It has a wide range of potential applications, some of which we investigated in this paper. We explored a healthcare dataset and here we hope to bring affirmative impact in medical applications. We also examined the modeling the vibration of an airplane and we expect it to contribute to the advancement of airplane designs and civil engineering. Furthermore, as many problems have arisen in sciences and engineering tied with complex PDE systems, we expect that our work has potential applicability in the enormous area such as climate forecasting, epidemics, molecular simulations, micro-mechanics, and modeling turbulent flows. PDEs can also be applied in the development of military equipment. As with all numerical methods, however, it is not a work of developing a technique to go to warfare, and we hope and encourage users of our model to concenter on the positive impact of this work.

\end{document}